\documentclass{article} %
\usepackage{iclr2027_conference,times}

\usepackage{amsmath,amsfonts,bm}

\def\eqref#1{equation~\ref{#1}}

\def\1{\bm{1}}

\DeclareMathAlphabet{\mathsfit}{\encodingdefault}{\sfdefault}{m}{sl}
\SetMathAlphabet{\mathsfit}{bold}{\encodingdefault}{\sfdefault}{bx}{n}

\newcommand{\E}{\mathbb{E}}

\newcommand{\R}{\mathbb{R}}

\usepackage{hyperref}
\usepackage{url}
\usepackage{enumitem}
\usepackage{graphicx,booktabs,xcolor}
\usepackage{wrapfig}
\usepackage{needspace}
\usepackage{amsthm}
\usepackage{algorithm}
\usepackage{algpseudocode}

\usepackage[font=small,skip=3pt]{caption}
\AddToHook{env/figure/begin}{\setlength{\parskip}{0pt}}
\AddToHook{env/table/begin}{\setlength{\parskip}{0pt}}
\AddToHook{env/wraptable/begin}{\setlength{\parskip}{0pt}}

\AtBeginDocument{\raggedbottom}

\newtheorem{proposition}{Proposition}

\theoremstyle{definition}

\theoremstyle{plain}

\hypersetup{
    colorlinks=true,
    citecolor=blue!60!black,
    linkcolor=blue!60!black,
    urlcolor=blue!60!black,
    pdftitle={Remember Less, Adapt Longer: Bounded-State In-Context Reinforcement Learning with Recurrent Summary Memory},
    pdfauthor={Jiawen Wang, Cheng Xue, Peng Zhang, Jochen Renz}
}

\title{Remember Less, Adapt Longer:\\
Bounded-State In-Context Reinforcement\\
Learning with Recurrent Summary Memory}

\author{Jiawen Wang, Cheng Xue, Peng Zhang, Jochen Renz \\
The Australian National University\\
\texttt{\{jiawen.wang, cheng.xue, peng.zhang, jochen.renz\}@anu.edu.au} \\
}

\newif\ifdraftnotes
\draftnotesfalse

\newcommand{\Hist}{\mathcal{H}}

\newcommand{\ourmethod}{Recurrent Summary Memory}
\newcommand{\method}{RSM}

\newcommand{\nocarry}{\method{} w/o memory}

\newcommand{\READ}{\textsc{Read}}
\newcommand{\RECORD}{\textsc{Record}}
\newcommand{\WRITE}{\textsc{Write}}
\newcommand{\dlt}{\delta}

\iclrfinalcopy %
\begin{document}

\maketitle

\begin{abstract}
In-context reinforcement learning (ICRL) enables agents to adapt to new tasks without parameter updates, relying on memory to retain what they learn from experience. For agents operating continually, existing approaches using memory trade off bounded cost against access to past experience.
Approaches using growing memories, such as full-history attention and accumulating summaries, incur storage and per-decision costs that increase with task length.
Fixed-size memories bound storage but restrict access: sliding windows discard older steps, while recurrent networks compress even recent interactions into a hidden state.
We introduce Recurrent Summary Memory (RSM), a Transformer policy that attends exactly to a short working buffer of recent interactions and carries a fixed-capacity summary of earlier experience.
When the buffer fills, RSM rewrites the summary from its previous contents and the buffer, then clears the buffer, so deployment memory and per-decision cost stay $O(1)$ in task length.
On long-running partially observable tasks, frozen RSM policies sustain their adapted performance for eight times the training horizon with about 1\% of full-history memory, ending ahead of every evaluated baseline while growing memories deteriorate.
When the task configuration changes mid-run, RSM re-adapts better than growing and recurrent memories.
Clearing the summary or blocking the policy's access to it removes its benefit, while swapping in another task's summary steers behaviour towards that task.
Together, these results show that our fixed-capacity summary can sustain in-context adaptation far beyond the training horizon.
\end{abstract}

\section{Introduction}
\label{sec:introduction}

An agent placed in an unfamiliar environment often cannot directly observe everything it needs to act, so it must use past experience to inform its decisions and adapt to the environment \citep{kaelbling1998planning}.
In Dark Key-to-Door \citep{laskin2022context}, for example, an agent that searches a dark room for a hidden key and then a door should reuse the key's location on every later attempt; in a maze with unfamiliar controls, an inferred mapping from actions to movements should support navigation throughout the task.
In-context reinforcement learning (ICRL) trains policies to adapt in this way: a single policy is trained across tasks to condition its decisions on its history of observations, actions, and rewards, enabling adaptation without test-time parameter updates \citep{laskin2022context,lee2023supervised,grigsby2024amago}.
Because its parameters remain fixed, memory must carry the task knowledge acquired during deployment \citep{duan2016rl,wang2016learning}.
For an agent that keeps operating, continued adaptation therefore depends on preserving useful discoveries.

Existing memory architectures trade off bounded cost against access to experience.
Full-history Transformers \citep{parisotto2020stabilizing,grigsby2024amago} let a policy revisit individual records, but their key/value caches and per-decision attention costs grow with task duration.
Sliding windows bound these costs by discarding older records, even when their information remains useful \citep{xiao2024efficient}.
Recurrent networks such as GRUs \citep{cho2014learning} maintain a fixed-size hidden state and provide strong baselines under partial observability, but compress even recent interactions \citep{heess2015memory,ni2022recurrent,morad2023popgym}.
Summary memory offers a way to separate immediate access from longer-term retention: recent records provide direct context, while a compact representation carries earlier experience \citep{rae2019compressive,ge2024incontext}.
Segment-level Transformers implement this separation with local attention and representations carried between segments \citep{dai2019transformer,hutchins2022block}: Recurrent Memory Transformers (RMTs) carry a fixed bank of memory tokens \citep{bulatov2022recurrent}, whereas Memo appends one summary block per segment, so its state still grows \citep{gupta2025memo}.

A fixed-capacity summary eliminates storage growth, but bounded storage alone does not guarantee sustained adaptation.
Once a segment is discarded, its useful evidence must survive not only the next summary, but every subsequent rewrite until it is needed.
Re-exploring the room may recover a forgotten key location, but not a cue presented only once \citep{ni2023memory}.
Conversely, when the hidden task changes, retaining obsolete evidence can obstruct adaptation.
The summary must therefore learn both to preserve relevant information and to revise it.
These demands grow beyond the training horizon, where memory undergoes more rewrites than in training.
Success on finite training tasks need not imply reliable memory during continued operation; indeed, finite and continuing memory tasks can reverse the relative performance of recurrent and attention-based agents \citep{pleines2025memory}.
We therefore ask: \emph{Can a Transformer policy sustain useful in-context adaptation beyond its local context and training horizon under a fixed deployment memory budget?}

We address this question with \textbf{\ourmethod{} (\method{})}, a Transformer policy with a recurrent-memory architecture.
\method{} combines exact attention over a short-term working buffer of the current segment's interactions with a fixed-capacity long-term summary.
Here, \emph{exact} means that each decision attends directly to the cached representation of every record in the buffer, not that the whole history is kept.
At each segment boundary, it writes a replacement summary from the previous summary and buffered steps, then clears the buffer, so only the new summary carries earlier experience into the next segment.
What the summary retains is learned only through the actor--critic losses, not set by token-retention rules \citep[StreamingLLM;][]{xiao2024efficient} or auxiliary reconstruction losses \citep[Compressive Transformers;][]{rae2019compressive}.
With fixed capacities, \method{}'s deployment state and per-decision computation are $O(1)$ in task length: in our experiments it holds 249.9\,KB per agent, whereas a full-history cache of the same backbone grows by 6\,KB per step, to about 5\,GB after a day of 10\,Hz control.
Such bounds matter for robots with fixed onboard memory, game agents acting over thousands of steps, and services running many agents in parallel.

To test whether this bounded memory sustains adaptation, we evaluate \method{} on four partially observable benchmarks: Dark Key-to-Door (location reuse; \citealp{laskin2022context}), CountRecall (evidence accumulation; \citealp{morad2023popgym}), MazeRunner (navigation with unknown action mappings; \citealp{grigsby2024amago}), and the passive T-Maze (one-time-cue recall; \citealp{ni2023memory}).
We compare against a full-history Transformer, a GRU, accumulating-summary Memo, and a state-matched sliding window, trained with the same learner, inputs, losses, and interaction budget.
On these benchmarks, a fixed-capacity summary sustains in-context adaptation far beyond the training horizon, answering our question affirmatively.
Specifically, our contributions are:
\begin{itemize}[noitemsep, topsep=2pt]
    \item \textbf{Bounded summary memory for online in-context RL.}
    We specify and evaluate \method{}, combining an exact segment-local working buffer with a fixed-capacity summary produced by learned write queries and a projected overwrite. \method{} is trained with AMAGO's \citep{grigsby2024amago} off-policy actor--critic losses through every carried write. Its deployment state and per-decision cost remain \(O(1)\) in task length (Section~\ref{sec:method}).
    \item \textbf{Sustained adaptation beyond the training horizon.}
    \method{} remains competitive at the training horizon; frozen \method{} policies maintain adapted performance over eight times that horizon and re-adapt better than growing and recurrent memories when the hidden task changes, whereas growing memories deteriorate (Section~\ref{sec:results-endpoint}).
    \item \textbf{Evidence that decisions depend on summary content.}
    Clearing the summary, transplanting another task's summary into it, and blocking the policy's access to it show that the policy uses its content; linear probes recover hidden task variables; and write-rule ablations show greater stability with overwriting than with additive updates (Sections~\ref{sec:results-mechanism}--\ref{sec:ablations}).
\end{itemize}

\section{Preliminaries and Problem Formulation}
\label{sec:problem}

\noindent\textbf{Task setting and in-context adaptation.}
A task $\xi$, such as a room layout or an action mapping, defines a finite-horizon POMDP
$\mathcal M_\xi=(\mathcal X,\mathcal O,\mathcal A,
P_\xi,\Omega_\xi,R_\xi,\rho_\xi,T_\xi)$ \citep{kaelbling1998planning}, with
hidden states $x_t\in\mathcal X$, observations $o_t\in\mathcal O$ drawn from the
observation function $\Omega_\xi(\cdot\mid x_t)$, actions $a_t\in\mathcal A$, dynamics $P_\xi$,
reward function $R_\xi$, initial-state distribution $\rho_\xi$ and horizon
$T_\xi$.
The agent plays each task as one \emph{outer task}, an episode of
$\mathcal M_\xi$ that may contain several physical attempts; an attempt reset is a transition of $P_\xi$ back to a
start state under the same $\xi$. At decision $t$ the policy conditions on
the history $\Hist_t=(o_0,a_0,r_1,o_1,\ldots,a_{t-1},r_t,o_t)$, where
$r_{t+1}=R_\xi(x_t,a_t)$ is the environment's native (unscaled) reward.
Every interface call is a step, including the reset-only call that starts
each new attempt: it takes no physical action, so the environment ignores
$a_t$, returns the new attempt's first observation and gives no reward
($r_{t+1}=0$), yet it counts towards the task length and the training budget. ICRL
learns one set of parameters $\theta$ for the policy and its encoders, shared
across all training tasks, so that the policy adapts to new tasks through
interaction \citep{duan2016rl,wang2016learning,grigsby2024amago}. It maximises
the expected discounted return $J_\gamma(\theta)$ over training tasks
$\xi\sim p_{\rm train}$ and the outer-task trajectories $\tau$ that $\pi_\theta$
generates in $\mathcal M_\xi$, with discount $\gamma\in(0,1]$:
\begin{equation}
 J_\gamma(\theta)=
 \E_{\xi\sim p_{\rm train},\,\tau\sim\pi_\theta,\mathcal M_\xi}
 \left[\sum_{t=0}^{T_\xi-1}\gamma^t r_{t+1}\right].
 \label{eq:icrl-objective}
\end{equation}
Because the return sums the rewards of every attempt in an outer task,
information gathered in early attempts pays off in later ones. At evaluation, $\theta$ is frozen and the policy runs on held-out
tasks disjoint from training; we report native, undiscounted metrics.
More details in Appendix~\ref{app:learning-objective}.

\noindent\textbf{History-dependent policies.}
All compared policies receive the same public input, the current observation
with the native RL$^2$ action--reward feedback \citep{duan2016rl}; task
identifiers, hidden layouts and evaluator labels never enter it. A stateless
timestep encoder $f_\phi$ embeds each step's record $u_t$ (with $a_{-1}$ and
$r_0$ set to zero), and a causal history encoder $F_\theta$ updates the
policy state:
\begin{equation}
 u_t=(o_t,a_{t-1},r_t),\qquad e_t=f_\phi(u_t),\qquad
 (m_t,s_t)=F_\theta(s_{t-1},e_t),\qquad a_t\sim\pi_\theta(\cdot\mid m_t).
 \label{eq:input}
\end{equation}
Here, $s_t$ denotes all history-dependent state that the policy retains for
future decisions, $m_t$ the actor--critic representation and $s_{-1}$ a
task-independent initial state; $\theta$ includes $\phi$. The state $s_t$ is
updated at every step and kept across attempt resets, and returns to $s_{-1}$
only when a new outer task begins.

\noindent\textbf{Bounded-state adaptation objective.}
We seek a policy that adapts through interaction while its
history-dependent deployment state respects a fixed budget $B$:
\begin{equation}
 \max_\theta J_\gamma(\theta)
 \quad\text{subject to}\quad
 \operatorname{bytes}(s_t)\leq B
 \quad\text{for all }\xi,\ \Hist_t,\ 0\leq t<T_\xi,
 \label{eq:budget}
\end{equation}
where $\operatorname{bytes}(s_t)$ is the memory needed to store $s_t$ at
implementation precision (its summary or hidden-state tensors, key/value
caches and counters, but not the parameters, workspace or training resources)
and $B$ does not depend on $T_\xi$: the deployment state must be $O(1)$ in the
task length, whereas a full-history Transformer, whose cache holds every
record, needs $O(T_\xi)$.
More details in Appendices~\ref{app:cost} and~\ref{app:state-others}.

\section{\ourmethod{}}
\label{sec:method}

\method{} (Figure~\ref{fig:method}) processes interaction history in segments of $C=32$ admitted records, using an exact segment-local working buffer and a recurrent summary of $M=4$ tokens. The same causal Transformer produces decision representations and the next summary. Learned write queries attend to the previous summary and the completed segment; a learned projection of their final-normalised outputs replaces the summary before the buffer is cleared.

\begin{figure}[!t]
  \centering
  \includegraphics[width=\linewidth]{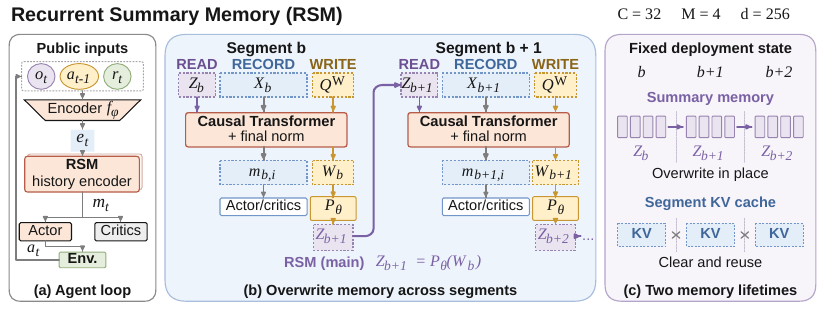}
  \caption{\textbf{Recurrent Summary Memory}.
  (a) The stateless encoder $f_\phi$ maps each public record to $e_t$;
  RSM supplies $m_t$ to the actor and critics.
  (b) $X_b=(e_{b,1},\ldots,e_{b,n_b})$ contains the segment's records.
  Each RECORD output $m_{b,i}$ uses $Z_b$ and the causal prefix $X_{b,1:i}$;
  normalised WRITE outputs $W_b$ yield $Z_{b+1}=P_\theta(W_b)$ for the next
  segment. Parameters are shared across segments; token embeddings are omitted.
  (c) Successive snapshots of the same bounded deployment state:
  the summary is replaced and the segment KV cache cleared (crosses).
  Counters are also retained; full-task training memory grows with task length.}
  \label{fig:method}
\end{figure}

\subsection{Working buffer and recurrent summary}
\label{sec:rsm-layout}

We partition the per-step embeddings $e_t=f_\phi(u_t)$ of Equation~\ref{eq:input} into
consecutive segments of $C=32$ records: record $t$ belongs to segment
$b=\lfloor t/C\rfloor$ at position $i=t-bC+1$, so that $e_{b,i}=e_t$ and
$m_{b,i}=m_t$. Segment $b$ receives a summary $Z_b\in\R^{M\times d}$ of
$M=4$ tokens (its rows) and the records $X_b=(e_{b,1},\ldots,e_{b,n_b})$, where
$1\leq n_b\leq C$ and only the last segment of a task can have $n_b<C$.
One causal Transformer \citep{vaswani2017attention} of width $d=256$, with
$L=3$ layers and eight heads,
supplies both the decision representation and the next summary:
\begin{equation}
 m_{b,i}=\mathcal R_\theta(Z_b,X_{b,1:i}),\qquad
 Z_{b+1}=\mathcal U_\theta(Z_b,X_b),\qquad Z_0=Z_{\rm init},
 \label{eq:summary-interface}
\end{equation}
where $Z_{\rm init}$ is a learned, task-independent initial summary. Both
maps share the slot layout
\begin{equation}
 [\;\READ(Z_b)\;|\;\RECORD(X_b)\;|\;\WRITE(Q^{\rm W})\;],
 \label{eq:layout}
\end{equation}
where $Q^{\rm W}\in\R^{M\times d}$ is a learned bank of write queries, and
\READ{}, \RECORD{} and \WRITE{} embed their arguments as rows of that role,
with role, memory-slot and segment-local position embeddings and, for
records, an input projection. The working
buffer is the per-layer key/value cache of these $M+C+M$ slots. All layers
apply ordinary causal attention over this slot order. The
final-normalised RECORD and WRITE outputs are the decision representations
$m_{b,i}$ and the rows of $W_b\in\R^{M\times d}$. Because the WRITE slots follow every RECORD
slot, $W_b$ cannot affect the current segment's decisions and reaches the
policy only through $Z_{b+1}$; READ and WRITE slots are not environment steps
and receive no direct RL loss. More details in
Appendices~\ref{app:embeddings} and~\ref{app:framework}.

\subsection{Overwrite update and training}
\label{sec:memory-updates}

At each boundary \method{} replaces the carried summary with the projected
WRITE outputs; the residual-write ablation \mbox{\method{}-R} adds them instead:
\begin{equation}
 Z_{b+1}=P_\theta(W_b)\quad\text{(\method{})},\qquad
 Z_{b+1}=Z_b+P_\theta(W_b)\quad\text{(\method{}-R)},
 \label{eq:write}
\end{equation}
where $P_\theta$ applies one affine map $\R^d\to\R^d$ to each of the $M$
rows, so that $\mathcal U_\theta(Z_b,X_b)=P_\theta(W_b)$ for \method{}. The
WRITE outputs depend on $Z_b$ through the READ rows, which makes the
overwrite recurrent without an additive identity path between summaries:
the writer must express both the retained information and the new evidence
in $M$ output tokens. Only $Z_{b+1}$ crosses a boundary: the working buffer
is cleared and the slot counter restarts. 
During cached interaction, a write is executed only when a full segment is followed by another record. Dense training also computes a final WRITE output, but it has no downstream decision loss.

\method{} is trained with AMAGO's off-policy actor--critic learner
\citep{grigsby2024amago}, shared by all compared policies: each update replays complete outer
tasks, recomputes their summaries from $Z_{\rm init}$ under the current
weights and backpropagates through every carried write
\citep{werbos1990backpropagation}, with no summary
target or auxiliary memory loss. Every later loss therefore reaches every
earlier write through $Z_{b+1}=P_\theta(W_b)$; under the overwrite, the
Jacobian $\partial Z_{b+2}/\partial Z_{b+1}$ passes only through attention to
the READ rows, whereas \mbox{\method{}-R} adds an identity term. 
Before each collection block, the parallel agents' states are rebuilt from their consumed public prefixes under the current weights.
More details in Appendices~\ref{app:framework}
and~\ref{app:learning-objective}.

\subsection{Deployment-state and summary bounds}
\label{sec:rsm-state}

\begin{proposition}[Bounded deployment state]
\label{prop:bounded-state}
For fixed $C$, $M$, $d$ and $L$, each agent of a frozen \method{} policy
retains its summary ($Md$ values), the working buffer's per-layer key/value
caches ($2L(C+2M)d$ values), a 32-bit occupied-slot counter and a 64-bit
segment counter. With its tensors in 32-bit floating point (FP32), this state
occupies
\begin{equation}
 B_{\rm deploy}=4\bigl[Md+2L(C+2M)d\bigr]+12
 \quad\text{bytes},
 \label{eq:state-bound}
\end{equation}
independent of task duration, for \mbox{\method{}-R} and \nocarry{} alike.
\end{proposition}
\textbf{Space complexity.} With the values of
Section~\ref{sec:rsm-layout}, $B_{\rm deploy}=249{,}868$ bytes per agent, so \method{}'s deployment state is $O(1)$ in the
task length $T_\xi$, as for a GRU \citep{cho2014learning}, whereas a
full-history Transformer of the same $L$ and $d$ caches keys and values for
every record ($2LT_\xi d$ values) and Memo \citep{gupta2025memo} for $M$
summary tokens per segment (about $2LMT_\xi d/C$ values), both $O(T_\xi)$.
Every \method{} query also attends to at most $C+2M=40$ cached slots, so
attention per decision is $O(1)$ in $T_\xi$ as well. The overwrite also keeps
every summary row within a fixed bound at any horizon, whereas the residual
write admits only a bound that grows with every write
(Proposition~\ref{prop:summary-range}). More details in
Appendix~\ref{app:memory-analysis}.

\section{Experiments}
\label{sec:experiments}

We ask (1)~whether
\method{} adapts as well as growing memories, at and far beyond the training
horizon (Section~\ref{sec:results-endpoint}); (2)~whether its
decisions use the summary (Section~\ref{sec:results-mechanism}); and
(3)~which design choices matter (Section~\ref{sec:ablations}).

\subsection{Experimental setup}
\label{sec:setup}

\begin{figure}[!t]
 \centering
 \makeatletter\def\@captype{table}\makeatother
 \caption{\textbf{Performance at the training horizon and over long tasks}
(settings in Section~\ref{sec:setup}). Training horizon: Key-to-Door (KD) doors
per 500 steps, CountRecall (CR) exact accuracy, MazeRunner (MR) fraction of
goals reached, T-Maze (TM) success. Long-running: doors or goals per 500 steps, in the last 500 of
4,000 steps (same task), per later layout or on the fourth map (new task). CR
evicted: accuracy on evidence that left the buffer. State: peak per agent. Bold: best compared memory.}
\label{tab:results}
{\footnotesize
\setlength{\tabcolsep}{2.6pt}
\begin{tabular}{lcccc@{\hspace{5pt}}cccccr}
\toprule
 & \multicolumn{4}{c}{Training horizon} & \multicolumn{4}{c}{Long-running, 4,000 steps} & & \\
\cmidrule(lr){2-5}\cmidrule(lr){6-9}
 & & & & & \multicolumn{2}{c}{same task} & \multicolumn{2}{c}{new task} & & \\
\cmidrule(lr){6-7}\cmidrule(lr){8-9}
 & KD $\uparrow$ & CR $\uparrow$ & MR $\uparrow$ & TM $\uparrow$ & KD $\uparrow$ & MR $\uparrow$ & KD $\uparrow$ & MR $\uparrow$ & CR evicted $\uparrow$ & State $\downarrow$ \\
\midrule
\textbf{\method{} (ours)} & 29.1 & 0.989 & 0.804 & \textbf{1.000} & \textbf{46.0} & \textbf{26.6} & \textbf{9.4} & \textbf{23.0} & 0.986 & 249.9\,KB \\
Full-history Transformer & 27.4 & \textbf{0.998} & 0.737 & \textbf{1.000} & 28.2 & 0.8 & 0.8 & 0.4 & \textbf{0.998} & 24.58\,MB \\
GRU & \textbf{29.6} & 0.930 & 0.643 & \textbf{1.000} & 40.6 & 19.1 & 1.6 & 16.1 & 0.937 & \textbf{3.1\,KB} \\
Memo & \textbf{29.6} & 0.956 & \textbf{0.863} & \textbf{1.000} & 10.6 & 1.8 & 0.5 & 2.0 & 0.939 & 3.24\,MB \\
Memo, fixed segments & 26.8 & 0.944 & 0.836 & \textbf{1.000} & 31.4 & 2.8 & 0.6 & 2.9 & 0.933 & 3.24\,MB \\
\midrule
\method{} w/o memory & 2.0 & 0.680 & 0.738 & 0.514 & 2.9 & 19.2 & 1.9 & 17.9 & 0.538 & 249.9\,KB \\
\method{}-R (residual) & 29.8 & 0.990 & 0.756 & 1.000 & 45.2 & 5.0 & 1.8 & 0.4 & 0.988 & 249.9\,KB \\
\bottomrule
\end{tabular}}

 \vspace{4pt}
 \makeatletter\def\@captype{figure}\makeatother
 \includegraphics[width=\linewidth,trim=0bp 1bp 0bp 2bp,clip]{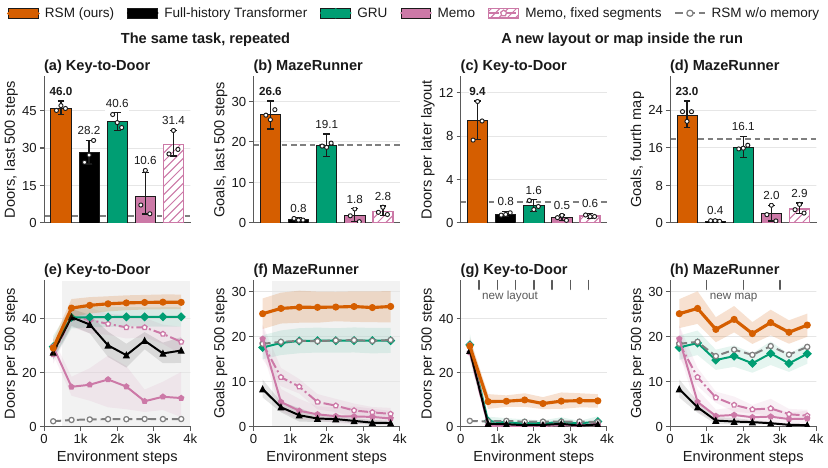}
 \caption{\textbf{Long-running tasks} (4,000 steps).
Columns: the same layout or map throughout, or a new one at intervals.
(a--d) Final-window scores, doors per later layout, or goals on the fourth
map; dots: seeds; dashed: \nocarry{}. (e--h) Scores per 500-step window;
shading marks steps beyond training, ticks mark task changes.}
 \label{fig:at-a-glance}
\end{figure}

\textbf{Environments.}
We use four partially observable benchmarks.
In \textbf{Dark Key-to-Door} (Key-to-Door below;
\citealp{laskin2022context,grigsby2024amago}), the
agent searches an $8\times8$ room for a hidden key and door over repeated
attempts in one 500-step task; \textbf{CountRecall} (POPGym;
\citealp{morad2023popgym}) deals 104 cards from four categories and asks at
every step how many of a queried category have been dealt; in
\textbf{MazeRunner} \citep{grigsby2024amago}, the agent reaches three ordered
goals in a hidden $15\times15$ maze with permuted actions in one 500-step
attempt; and the \textbf{passive T-Maze} \citep{ni2023memory} requires
recalling a cue shown only at the first step of a 128-cell corridor.

\textbf{Baselines and Ablations.}
All policies share the stateless timestep encoder, the public inputs and the
AMAGO off-policy learner; the Transformers share width 256, three layers and eight
heads. We compare \method{} (Section~\ref{sec:method}) with the \textbf{full-history
Transformer}, whose key--value cache grows with every record; AMAGO's
recurrent \textbf{GRU}, whose state is fixed; and \textbf{Memo}
\citep{gupta2025memo}, which appends four summary tokens per 32-record segment,
both as published (jittered segment lengths) and with \textbf{fixed segments}.
We call the full-history Transformer and Memo the
\emph{growing memories}.
The \textbf{sliding window} is a summary-free Transformer attending only to
the last $W=40$ records at every layer, the largest window within \method{}'s
per-agent state.
The ablations train \textbf{\method{} w/o memory}, which resets the summary at
every segment boundary; \textbf{\mbox{\method{}-R}}, which adds the write to
the carried summary instead of overwriting it (Equation~\ref{eq:write});
\method{} with $M\in\{1,4,8,16\}$ summary tokens on Key-to-Door; and, on
CountRecall, \textbf{\method{} detach}, whose summary is detached at every
boundary so that no later loss reaches the write.

\textbf{Evaluation Protocol.}
Every policy is trained three times, with seeds
42, 100 and 2026, for 8M steps each, and evaluated at its final checkpoint by
greedy rollouts on 256 held-out \emph{confirmation} tasks per benchmark, whose
layouts, card streams and mazes never occur in training. The metrics are doors
per 500 steps on Key-to-Door, exact accuracy on CountRecall (\emph{CR evicted}:
queries whose evidence has left the buffer), the fraction of
the three goals reached on MazeRunner, and success on the T-Maze, each first
at the \emph{training horizon}, the length of a training task. CountRecall and
the T-Maze have fixed-length tasks; Key-to-Door and MazeRunner also run for
4,000 steps, eight times that horizon, counting successes per 500 steps:
in the \emph{same task} the layout or map stays fixed while attempts continue,
and in a \emph{new task} a new layout arrives every 500 steps, or a new maze
with the same action mapping every 1,000, with memory carried.
A change of layout or maze within such a run is not a new outer task in the
sense of Section~\ref{sec:problem}: no memory is reset during the 4,000 steps. We report means over tasks and seeds and
compare policies by paired differences on the same tasks and seeds: a
difference is \emph{consistent} when it exceeds a practical margin $\dlt$
(1 door, 0.05 accuracy, 0.1 goal fraction or one goal per 500 steps, and 0.25
success), its 95\% joint seed/task bootstrap interval
\citep{agarwal2021precipice} excludes zero and all three seeds agree in sign,
and \emph{equivalent} when the whole interval lies within $\pm\dlt$; any
other difference is \emph{no consistent difference}. Bands and whiskers are
95\% intervals; dots mark seeds. More details in Appendices~\ref{app:implementation},
\ref{app:benchmarks} and~\ref{app:learning}.

\subsection{Adaptation at and beyond the training horizon}
\label{sec:results-endpoint}

\textbf{At the training horizon.}
\method{} is competitive at the training horizon and pulls ahead over long
tasks (below). Some baselines score higher in
mean, but none is consistently better: the full-history Transformer reaches
0.998 against 0.989 accuracy on CountRecall, within the equivalence margin,
and Memo reaches a goal fraction of 0.863 against 0.804 on MazeRunner, ahead on
every seed but by a mean gap below $\dlt$. \method{} is equivalent to the
full-history Transformer on CountRecall, MazeRunner and the T-Maze and
consistently better than the GRU on CountRecall and MazeRunner; on Key-to-Door the intervals
are too wide to separate it from the growing memories and the GRU
(Tables~\ref{tab:results} and~\ref{tab:contrasts}). These scores reflect adaptation within a task:
on Key-to-Door, door success rises from about 0.15 on the first attempt to
above 0.8 on the ninth for every memory except \nocarry{}
(Appendix~\ref{app:in-context}). Its state stays fixed while the full-history cache and Memo's store
grow; after 4,000 steps it is about 1\% of the full-history cache
(Figure~\ref{fig:state}).

\label{sec:results-horizon}
\Needspace{14\baselineskip}
\begin{wrapfigure}{r}{0.47\linewidth}
\vspace{-12pt}
\centering
\includegraphics[width=\linewidth]{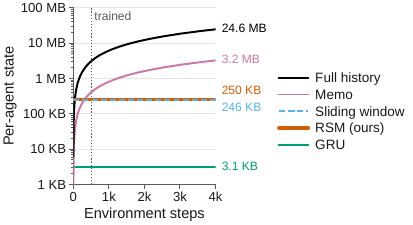}
\vspace{-7pt}
\caption{\textbf{Per-agent state} against task length.}
\label{fig:state}
\vspace{-10pt}
\end{wrapfigure}
\textbf{Over a long task.}
In the same task, \method{} and the GRU sustain their adapted performance over the whole run, \method{} at a higher level and, in the last window, ahead of
every compared memory on every seed (Figure~\ref{fig:at-a-glance}a,b,e,f;
Table~\ref{tab:results}), so its advantage widens as the task runs on. The
growing memories fall from their early-run rate, Memo most steeply, and on the repeated MazeRunner map they collapse towards zero goals;
they also meet longer contexts than in training, while \method{}'s context
stays at 40 slots (Section~\ref{sec:rsm-state}). The causes differ: on one
MazeRunner map, wiping the memory at every attempt prevents Memo's decline but
not the full-history Transformer's, so retained history alone does not explain
the latter (Appendix~\ref{app:extended}).

\begin{wraptable}{r}{0.49\linewidth}
\vspace{-12pt}
\centering
\caption{\textbf{A state-matched sliding window} (the last 40 records per layer,
246.1\,KB per agent).}
\label{tab:window}
\footnotesize
\setlength{\tabcolsep}{2.6pt}
\begin{tabular}{lcccc}
\toprule
 & \multicolumn{2}{c}{Training horizon} & \multicolumn{2}{c}{Last 500 of 4,000} \\
\cmidrule(lr){2-3}\cmidrule(lr){4-5}
 & KD & MR & KD & MR \\
\midrule
\method{} & 29.1 & 0.804 & 46.0 & 26.6 \\
window & 24.1 & 0.750 & 4.9 & 2.6 \\
\bottomrule
\end{tabular}
\vspace{-8pt}
\end{wraptable}

\textbf{A state-matched sliding window.}
The sliding window, with about the same deployment state as \method{}, trails it
consistently on Key-to-Door at the training horizon, where every position was
seen in training, so evidence older than 40 records matters; on MazeRunner the
two show no consistent difference. Over 4,000 steps the window falls to a few
doors or goals per window (Table~\ref{tab:window}), but it then also meets
untrained positions, so this comparison is operational rather than a clean
test of summarisation.

\textbf{After the task changes.}
In the new task (Table~\ref{tab:results};
Figure~\ref{fig:at-a-glance}c,d,g,h), \method{} drops after each change and
recovers once it finds the new layout or map, largely on MazeRunner and in
about a third of Key-to-Door layouts, yet stays above every other memory on
every later layout, map and seed. On Key-to-Door the others average under two
doors per layout; on MazeRunner the GRU and \nocarry{} stay lower and the
growing memories collapse. More details in Appendices~\ref{app:long-protocol}
and~\ref{app:extended}.

\Needspace{16\baselineskip}
\subsection{Do the decisions use the summary?}
\label{sec:results-mechanism}

\begin{figure}[H]
 \centering
 \includegraphics[width=\linewidth]{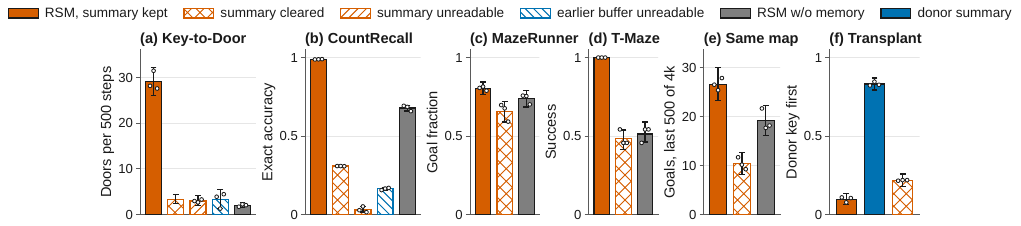}
 \caption{\textbf{Interventions on the frozen \method{} policy.} (a--e) The
summary kept or cleared at every boundary, beside \nocarry{}, and in (a,b) the
summary or the earlier buffer unreadable to the decisions ($\beta=\infty$); (e) MazeRunner
goals in the last 500 of 4,000 steps on one map. (f) Key-to-Door tasks whose
next attempt visits the donor's key first, after a transplant. Dots: seeds.}
 \label{fig:interventions}
\end{figure}

\textbf{Clearing, transplanting and blocking.}
We test causally whether the decisions depend on the carried summary
(Figure~\ref{fig:interventions}). Resetting \method{}'s summary to
$Z_{\rm init}$ at every boundary, with the working buffer intact, removes the
benefit of the summary on every benchmark and every seed, and leaves the policy
well below \nocarry{} on CountRecall and the long MazeRunner task and at its
fixed-turn level on the T-Maze: the trained policy
does not fall back on the buffer-only strategy that \nocarry{} learns
(Figure~\ref{fig:interventions}a--e). To transplant, we replace each
task's summary at one boundary with the summary that a different task, the
\emph{donor}, carried there; behaviour then follows the donor task: on Key-to-Door the next attempt visits the donor's key cell first
in most tasks (Figure~\ref{fig:interventions}f), and on CountRecall the policy
answers with the donor's counts until the next rewrite. As attention weights alone do not explain a decision
\citep{jain2019attention}, we also intervene on them: adding $-\beta$ to every
decision's attention scores on the four summary keys leaves the summary
carried but unreadable at $\beta=\infty$, which reduces \method{} to the level
of \nocarry{} on Key-to-Door and below it on CountRecall
(Figure~\ref{fig:interventions}a,b), with a loss growing with $\beta$.
Blocking the earlier buffer records costs about as much: the decisions depend
on both memories.

\Needspace{24\baselineskip}
\begin{wrapfigure}{r}{0.47\linewidth}
\vspace{-12pt}
\centering
\includegraphics[width=\linewidth]{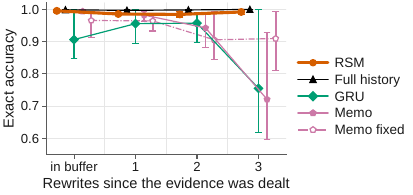}
\vspace{-11pt}
\caption{\textbf{CountRecall retention}: accuracy by the number of summary
rewrites since the queried evidence was dealt.}
\label{fig:retention-age}
\vspace{2pt}
\includegraphics[width=\linewidth]{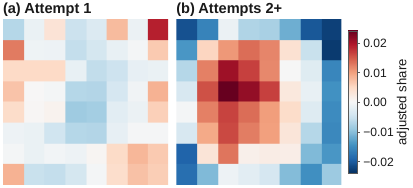}
\vspace{-11pt}
\caption{\textbf{Where the summary is read} in the Key-to-Door room: summary
attention adjusted for segment position, by the agent's cell.}
\label{fig:room-attention}
\vspace{-10pt}
\end{wrapfigure}
\textbf{Retention and Attention.}
On CountRecall, we group queries by how often the summary has been rewritten
since the queried category was first dealt (Figure~\ref{fig:retention-age}).
Queries whose evidence has passed through up to three rewrites are answered
about as accurately as by the full-history Transformer, which keeps every card,
while the GRU and both Memo variants are less accurate in the last segment. A linear
read-out (ridge regression) from the summary recovers every CountRecall count
almost perfectly and the hidden key's position on Key-to-Door well, on every
seed: four tokens hold the task's hidden variables in a readable form. The
decisions also attend to the summary far more than the same slots in \nocarry{};
on Key-to-Door this extra attention is flat across the room on the
first attempt, while the agent searches, and on later attempts is higher in
the room's interior than along its walls (Figure~\ref{fig:room-attention}):
once the agent has something to recall, how much it reads the summary varies
slightly with where it is.
More details in Appendices~\ref{app:interventions}--\ref{app:attention}.

\subsection{Ablations}
\label{sec:ablations}

\textbf{Carrying and training the summary.}
Removing the carried summary, keeping only the 32-record buffer, loses
most of the adaptation (Table~\ref{tab:results}, \nocarry{}), consistently on
every seed on Key-to-Door, CountRecall and the T-Maze, and the loss widens over
the 4,000-step Key-to-Door task. On MazeRunner one 500-step attempt is
largely solvable from the buffer, so the summary adds less than $\dlt$ at the
training horizon, but several goals per window over longer tasks. How the
write is trained also matters: detaching the summary at every boundary
lowers \method{}'s CountRecall accuracy
on every seed (0.914 against 0.989), yet keeps most of the gain over \nocarry{}
(0.680). Training each write through the decisions that later read it thus
improves what is kept.

\textbf{Overwrite versus additive write.}
\mbox{\method{}-R} stays within
$\dlt$ of \method{} in mean at the training horizon and
over the 4,000-step Key-to-Door task, but over many MazeRunner attempts on one
map it declines below \nocarry{} and on the map-change task it
collapses by the fourth map, on every seed, while after a Key-to-Door layout
change it adapts again no better than \nocarry{} (Table~\ref{tab:results}). Wiping its memory at every attempt boundary
preserves its rate, so the decline comes from what it carries. Its summaries
drift in one direction until the first and the last are anti-correlated, at a
norm tens of times larger than \method{}'s: a drift that
Proposition~\ref{prop:summary-range} bounds only for the overwrite.

\textbf{Summary length.}
Varying the summary from one to sixteen tokens on Key-to-Door, with
buffer, backbone and budget fixed, changes adaptation by no consistent
margin (Table~\ref{tab:summary-length}): the paired differences from four tokens change
sign across seeds at the training horizon, and in the last window every length
completes 45 to 46 doors. Even one token shows no consistent difference from four, and a larger
summary buys no further adaptation on this task. More details in
Appendices~\ref{app:representation} and~\ref{app:ablations}.

\begin{table}[H]
\centering
\caption{\textbf{Summary length on Key-to-Door} (4,000-step tasks; 95\% intervals).}
\label{tab:summary-length}
\footnotesize
\setlength{\tabcolsep}{4pt}
\begin{tabular}{lcccc}
\toprule
 & $M=1$ & $M=4$ (\method{}) & $M=8$ & $M=16$ \\
\midrule
State per agent & 209.9\,KB & 249.9\,KB & 303.1\,KB & 409.6\,KB \\
Doors, steps 1--500 & 29.6 [26.5, 33.0] & 29.1 [25.9, 32.4] & 30.8 [28.5, 33.1] & 30.1 [27.8, 32.4] \\
Doors, steps 3,501--4,000 & 45.3 [42.2, 48.2] & 46.0 [43.3, 48.6] & 45.3 [42.5, 48.1] & 45.9 [43.1, 48.7] \\
\bottomrule
\end{tabular}
\end{table}

\newpage

\section{Related Work}
\label{sec:related}

\paragraph{In-context reinforcement learning.}
Recurrent meta-RL first achieved in-context adaptation through hidden-state updates
\citep{duan2016rl,wang2016learning}, and PEARL and VariBAD infer a belief
over the task from context \citep{rakelly2019pearl,zintgraf2020varibad}.
Sequence models extend the idea to longer histories: Algorithm Distillation
imitates the learning histories of RL algorithms \citep{laskin2022context},
Decision-Pretrained Transformers predict optimal actions from interaction
data \citep{lee2023supervised}, and AMAGO and ReLIC learn in context through
off-policy and on-policy updates over long sequences
\citep{grigsby2024amago,elawady2024relic}. How well such agents remember is
harder to measure: memory is easily confounded with exploration and
credit assignment \citep{pasukonis2023memory,ni2023memory,cherepanov2026unraveling},
finite and endless versions of a task can reverse the ranking of
recurrent and Transformer memories \citep{pleines2025memory}, and retaining
information differs from rewriting it when the task changes
\citep{shchendrigin2026memory}. Our evaluation follows these findings: we ran
frozen policies far beyond the training horizon, changed the hidden task
mid-run, and intervened on the summary directly.

\paragraph{Memory for sequence-based control.}
Like recurrent networks, state-space models and gated linear attention
keep a fixed-size state \citep{lu2023structured,pramanik2024agalite}, and
Stable Hadamard Memory calibrates each update to erase what is no longer
needed \citep{le2025stable}. Learned summaries compress long contexts in
language models \citep{chevalier2023adapting,ge2024incontext}, and language-model
and RL memories have also been trained with reconstruction or predictive losses \citep{rae2019compressive,wayne2018unsupervised}.
In RL, RATE and ELMUR carry segment memory offline, ELMUR
detaching it between segments \citep{cherepanov2026recurrent,cherepanov2026elmur},
and Memo, the closest online method, compresses each segment into summary
tokens and backpropagates through them, but keeps every summary block, so its
state grows with the task \citep{gupta2025memo}. Memo also prompts each write
with learned summary embeddings and reports a fixed-size, RMT-style variant,
finding that accumulating summaries better retains relevant past experience; we
instead test whether a fixed-capacity summary, trained with an overwrite,
sustains adaptation far beyond the training horizon. Like Recurrent
Memory Transformers \citep{bulatov2022recurrent}, \method{} carries a fixed
bank of memory tokens between segments. RMTs are trained with token-prediction
losses on language and algorithmic tasks and place the carried memory
at both their read and write positions; \method{} instead writes with learned,
history-independent queries, as Memo does,
overwrites its summary with their projected, normalised outputs, and is
trained online through actor--critic losses.

\section{Discussion and Conclusion}
\label{sec:conclusion}

We introduced Recurrent Summary Memory, a Transformer policy that attends
exactly to a short working buffer and carries a fixed-capacity summary,
overwritten at every segment boundary and trained only through actor--critic
losses. With 249.9\,KB of deployment state per agent, about 1\% of a
full-history cache after 4,000 steps, \method{} remains competitive with
growing memories at the training horizon. Over tasks eight times longer, it
sustains its adapted performance and re-adapts better than the growing and
recurrent memories when the hidden task changes, while the growing memories deteriorate.
Memories that perform alike at the training horizon can thus diverge in
continued operation, so finite-horizon scores alone are an incomplete test of
memory for long-running agents. Clearing, transplanting and blocking the
summary show that the decisions depend on what it carries, and the overwrite
provably bounds the summary at any horizon, whereas an additive write drifts.
These results suggest that a fixed-capacity memory can keep an agent adapting
far beyond its training horizon, provided it is trained to decide what to keep
at every rewrite.

Three limits remain. We trained only on tasks of at most 500 steps; only the
deployment memory is bounded, since each training update replays whole tasks
and backpropagates through every summary write; on Key-to-Door, \method{}
learns more slowly than every other memory that learns the task and, after a
layout change, recovers in only about a third of layouts. 
Future work can also bound training memory, for example by backpropagating through only
the most recent summaries: on CountRecall, cutting the gradient at every
boundary still kept most of the summary's benefit (Section~\ref{sec:ablations}).
Update rules could also learn when to keep or revise evidence over
far longer horizons, and \method{} could be extended to visual and embodied
agents, such as robots running for hours on fixed onboard memory.

\label{end:main}

\subsection*{AI use statement}
In this work, we used generative AI tools to assist with developing and executing training and evaluation scripts. We did not use generative AI tools to generate research ideas; the research ideas were developed independently by the authors. Additionally, we used generative AI tools to polish the manuscript for clarity, concision, and readability. We have reviewed all AI-assisted work, including the manuscript revisions and the training and evaluation scripts. We take responsibility for the final content of this work, including all text, code, claims, and experimental results produced with the aid of generative AI.

\subsection*{Reproducibility statement}
Section~\ref{sec:method} defines the memory operator, and
Appendices~\ref{app:framework}--\ref{app:memory-analysis} provide pseudocode,
the learning objective and the memory analysis.
Appendix~\ref{app:implementation} reports the configurations, training budgets,
task splits and evaluation procedures.
A repository containing the implementation, pinned environment,
experiment configurations, tests, summary evaluation records, and scripts and
notebooks used to regenerate the reported figures and tables is available at
\url{https://github.com/Jiaweeeeeen/RSM-ICLR}.

\bibliography{references}
\bibliographystyle{iclr2027_conference}
\clearpage
\appendix
\noindent\textbf{Appendix overview.} Appendix~\ref{app:implementation}
describes the compared memories, their training settings and compute;
Appendices~\ref{app:framework} and~\ref{app:learning-objective} give
\method{}'s pseudocode and learning objective;
Appendix~\ref{app:memory-analysis} proves the memory bounds and accounts for
every memory's state; Appendix~\ref{app:benchmarks} specifies the benchmarks,
long-running evaluations and statistics; and Appendix~\ref{app:experiments}
reports additional results.

\section{Compared Memories and Training}
\label{app:implementation}

All compared memories receive the same public record: the current observation,
previous action and previous reward, encoded by the same stateless timestep
encoder with no observation bypass. Each agent carries its own
memory state and resets it only at an outer-task boundary. The Transformers
have width 256, three layers and eight heads, and the GRU matches their width
and depth; \method{} uses $C=32$ and $M=4$. Within each benchmark, every
memory is trained by the same learner (Appendix~\ref{app:learning-objective})
with the same replay, exploration and update schedule
(Appendix~\ref{app:training}), in FP32\@. Table~\ref{tab:cells} lists the
compared memories with their identifiers in our code and their parameter
counts, which range over the four benchmarks because their input and action
sizes differ; the GRU has about 1.3M fewer parameters than the Transformers.

\begin{table}[H]
\centering\scriptsize
\caption{Compared memories, their code identifiers, what each carries
forward and its parameter count over the benchmarks it runs on.}
\label{tab:cells}
\begin{tabular}{lllr}
\toprule
Memory & Code identifier & Carried forward & Parameters \\
\midrule
\method{} (ours) & \texttt{raw\_summary} & $P_\theta(W)$ at each boundary & 3.18--3.39M \\
Full-history Transformer & \texttt{full\_context} & Every record's keys and values & 3.11--3.32M \\
GRU & \texttt{full\_gru} & Recurrent hidden state & 1.82--2.03M \\
Memo & \texttt{memo} & Every summary block; jittered segments & 3.11--3.32M \\
Memo, fixed segments & \texttt{memo\_fixed} & Every summary block & 3.11--3.32M \\
Sliding window & \texttt{raw\_window} & The last $W=40$ records per layer & 3.12M \\
\nocarry{} & \texttt{raw\_segment} & $Z_{\rm init}$ at each boundary & 3.18--3.39M \\
\method{}-R & \texttt{raw\_summary\_residual} & $Z+P_\theta(W)$ at each boundary & 3.18--3.39M \\
\method{} detach & \texttt{raw\_summary\_detach} & $P_\theta(W)$, gradient stopped & 3.39M \\
\bottomrule
\end{tabular}
\end{table}

\subsection{Memo adaptation}
\label{app:memo}

Our comparator follows the author-linked Memo implementation at revision
\texttt{9e7044f}, adapted to the shared backbone and learner. Each segment
produces four summary tokens. Memo samples a segment length from 26--38
once per learner forward and uses 32 during interaction; the fixed-segment
variant uses 32 in both settings. Both retain all earlier summary blocks
and propagate gradients through them. This comparison matches the present
study's conditions rather than reproducing Memo's published task,
architecture, and budget in full.

\subsection{Sliding window}
\label{app:window}

The sliding window runs the full-history Transformer's backbone under a
banded causal mask: every layer attends only to the last $W=40$ records,
through a rolling cache of 40 slots per layer during interaction. $W=40$ is
the largest window whose state fits within \method{}'s per-agent state
(Appendix~\ref{app:state-others}). Its records keep their absolute positions
in the task, so beyond the 500-step training tasks it meets positions it never
trained on. It runs on Key-to-Door and MazeRunner with the same learner,
8M-step interaction budget and seeds as every other memory.

\subsection{RSM embeddings and controls}
\label{app:embeddings}

\paragraph{Token embeddings.}
Let $p_k$ be the fixed sinusoidal embedding of slot $k$, $s_j$ a learned
embedding of memory row $j$, and $\rho^{\rm R},\rho^{\rm X},\rho^{\rm W}$
the learned role embeddings. Before embedding dropout, the rows are
\begin{align}
 \widetilde z_{b,j}&=Z_b[j]+s_j+\rho^{\rm R}+p_{j-1},
 &&1\leq j\leq M,\nonumber\\
 \widetilde e_{b,i}&=A_\theta(e_{b,i})+\rho^{\rm X}+p_{M+i-1},
 &&1\leq i\leq C,\nonumber\\
 \widetilde q_j&=Q^{\rm W}[j]+s_j+\rho^{\rm W}+p_{M+C+j-1},
 &&1\leq j\leq M.
 \label{eq:rsm-embeddings}
\end{align}
The input projection is $A_\theta$. READ and WRITE rows share memory-slot
embeddings but use distinct roles and positions. Both $Z_{\rm init}$
and $Q^{\rm W}$ are learned; the affine map
$P_\theta:\R^d\to\R^d$ projects each final-normalised WRITE output.
Parameters are shared across segments, and training retains gradients
through the carried summary.

\paragraph{Controls.}
\nocarry{} preserves \method{}'s writer modules and boundary schedule, but
discarded writes receive no gradient from later segments. This separately
trained control differs from clearing the summary of a frozen \method{}
policy at evaluation.

\subsection{Training settings and compute}
\label{app:training}

Training runs 1,000 epochs; each collects 500 steps from each of 16 agents and
then takes 128 gradient updates, for 8M environment steps, reset-only steps
included. Saved counters verify this budget. Replay retains complete outer
tasks, and before each collection block the agents' states are rebuilt under
the current weights (Appendix~\ref{app:framework}). On MazeRunner, following
AMAGO's MazeRunner recipe, replay also relabels goals in hindsight for every
memory: a replayed task has a random number of its unreached goals replaced
by positions the agent reached, and its rewards recomputed; the other
benchmarks use no relabeling. Every 50 epochs the greedy
policy is scored on the 64 development tasks (Appendix~\ref{app:learning}); the
reported results use the final checkpoint, with no checkpoint selection.
Table~\ref{tab:hyper} lists the settings. Exploration anneals over each
agent's first 50,000 steps, as in AMAGO: $\epsilon$-greedy from 1.0 to 0.05;
Key-to-Door's bilevel schedule also lowers $\epsilon$ within each task, from
1.0 to 0.8 at the start of training and from 0.05 to 0.01 once annealed; and
the T-Maze schedule holds $\epsilon$ at $0.5/L$ inside the corridor and uses
the annealed value at the first decision and at the last decisions before the
junction.

\begin{table}[H]
\centering\small
\caption{Training settings, shared by every memory within a benchmark.}
\label{tab:hyper}
\begin{tabular}{lcccc}
\toprule
 & Key-to-Door & CountRecall & MazeRunner & T-Maze \\
\midrule
Optimiser & \multicolumn{4}{c}{AdamW; learning rate $10^{-4}$; weight decay $10^{-3}$} \\
Warm-up; gradient-norm clip & \multicolumn{4}{c}{128 updates; 1.0} \\
Epochs $\times$ steps $\times$ agents & \multicolumn{4}{c}{$1{,}000\times500\times16=8$M environment steps} \\
Updates per epoch; batch & \multicolumn{4}{c}{128; 16 complete tasks} \\
Replay capacity & \multicolumn{4}{c}{10,000 complete tasks} \\
Longest training task (steps) & 500 & 103 & 500 & 257 \\
Target-network rate $\tau$ & 0.003 & 0.0025 & 0.003 & 0.0025 \\
Exploration ($\epsilon$-greedy) & bilevel & standard & standard & T-Maze \\
Hindsight goal relabeling & no & no & yes & no \\
\bottomrule
\end{tabular}
\end{table}

\paragraph{Compute.}
The paper reports 102 fits: seven memories on four benchmarks, the sliding
window on two, \method{} detach and three further summary lengths, each with
three seeds. Every fit trained in FP32 on one GPU at a time, on NVIDIA RTX
3090, RTX A6000, A30, L40S, RTX 6000 Ada, RTX PRO 6000 Blackwell and A100
GPUs, and on RTX 2080 Ti GPUs for the $M=1$ arm. The records hold at least 880
fit-hours of training time, excluding evaluation. Seventy-two fits were
resumed at least once, often on another GPU model, and several fits sometimes
shared one GPU\@. These durations are therefore not hardware-normalised and are
not used to compare training efficiency across memories.

\begingroup
\raggedbottom
\section{RSM Pseudocode}
\label{app:framework}

Algorithm~\ref{alg:rsm-step} implements cached interaction for one agent;
Algorithm~\ref{alg:rsm-dense} gives dense training and agent-state
reconstruction. Both encode the public record $u_t=(o_t,a_{t-1},r_t)$
with the stateless map $e_t=f_\phi(u_t)$. State
$s=(Z,\mathcal K,\ell,b)$ contains the summary, per-layer key/value
caches of $C+2M$ slots, the occupied-slot counter, and the segment counter
(the number of completed writes).
$\operatorname{Append}_\theta$ processes embedded rows in causal order
(Equation~\ref{eq:rsm-embeddings}), updates the caches and slot count,
and returns final-normalised outputs. The boundary operator selects the
write rule:
\begin{equation}
 \operatorname{Update}_{\omega,\theta}(Z,W)=
 \begin{cases}
 P_\theta(W), & \omega=\mathrm{overwrite}\quad\text{(\method{})},\\
 Z+P_\theta(W), & \omega=\mathrm{residual}\quad\text{(\method{}-R)},\\
 Z_{\rm init}, & \omega=\mathrm{no\ carry}\quad\text{(\nocarry{})}.
 \end{cases}
 \label{eq:pseudocode-update}
\end{equation}

\begin{algorithm}[!htbp]
\caption{\method{} interaction with lazy boundary writes (one agent)}
\label{alg:rsm-step}
\small
\begin{algorithmic}[1]
\Require Parameters $\theta$ (including $\phi$), rule $\omega$, capacities $C,M$
\Function{ResetAgent}{}
  \State \Return $(Z_{\rm init},\Call{EmptyCache}{C+2M},0,0)$
\EndFunction
\Function{Advance}{$u,s=(Z,\mathcal K,\ell,b)$}
  \If{$\ell=0$} \Comment{The READ pass is pending}
    \State $(-,\mathcal K,\ell)\gets\operatorname{Append}_\theta(\mathcal K,\ell,\operatorname{READ}(Z))$
  \EndIf
  \If{$\ell=M+C$} \Comment{Write before admitting the next record}
    \State $(W,\mathcal K,\ell)\gets\operatorname{Append}_\theta(\mathcal K,\ell,\operatorname{WRITE}(Q^{\rm W}))$
    \State $Z\gets\operatorname{Update}_{\omega,\theta}(Z,W)$
    \State $\mathcal K\gets\Call{EmptyCache}{C+2M}$; $\ell\gets0$; $b\gets b+1$
    \State $(-,\mathcal K,\ell)\gets\operatorname{Append}_\theta(\mathcal K,\ell,\operatorname{READ}(Z))$
  \EndIf
  \State $e\gets f_\phi(u)$
  \State $(m,\mathcal K,\ell)\gets\operatorname{Append}_\theta(\mathcal K,\ell,\operatorname{RECORD}(e,\ell))$
  \State \Return $m,(Z,\mathcal K,\ell,b)$
\EndFunction
\end{algorithmic}
\end{algorithm}

The action is selected from the RECORD representation $m$; READ and WRITE
outputs produce no actions or direct RL losses. A segment that has just
received its $C$th record keeps its write pending until the next record
arrives. \textsc{ResetAgent} is called only at an outer-task reset, so
physical-attempt resets preserve the state.

\begin{algorithm}[!htbp]
\caption{Dense \method{} unroll and agent-state rebuild}
\label{alg:rsm-dense}
\small
\begin{algorithmic}[1]
\Require Parameters $\theta$, rule $\omega$
\Function{DenseUnroll}{$U=(u_0,\ldots,u_{N-1}),v$} \Comment{$v$: record validity}
  \State $E\gets f_\phi(U)$; $Z\gets Z_{\rm init}$; $O\gets[\,]$
  \State Right-pad $E$ and $v$ to a multiple of $C$; padded records have $v=0$
  \For{each consecutive $C$-record block $(E_b,v_b)$}
    \State $X\gets[\operatorname{READ}(Z)\mid\operatorname{RECORD}(E_b)\mid
                   \operatorname{WRITE}(Q^{\rm W})]$
    \State $H\gets\operatorname{Norm}_\theta(\operatorname{Transformer}_\theta(X))$
      \Comment{Causal; padded records masked as keys}
    \State Append the RECORD rows $H^{\rm R}$ to $O$
    \State $Z\gets\operatorname{Update}_{\omega,\theta}(Z,H^{\rm W})$ \Comment{Keep the gradient path}
  \EndFor
  \State \Return $v\odot O_{0:N-1},Z$
\EndFunction
\Statex
\Function{Rebuild}{$U^{\rm proc}=(u_0,\ldots,u_{n-1})$} \Comment{No gradients}
  \State $s\gets\Call{ResetAgent}{}$
  \If{$n=0$} \State \Return $s$ \EndIf
  \State $q\gets\lfloor(n-1)/C\rfloor$ \Comment{Segments whose writes have occurred}
  \If{$q>0$}
    \State $(-,Z)\gets\Call{DenseUnroll}{U^{\rm proc}_{0:qC-1},\mathbf 1_{qC}}$;
      $s\gets(Z,\Call{EmptyCache}{C+2M},0,q)$
  \EndIf
  \For{$t=qC,\ldots,n-1$}
    \State $(-,s)\gets\Call{Advance}{u_t,s}$ \Comment{Re-encode the open segment}
  \EndFor
  \State \Return $s$
\EndFunction
\end{algorithmic}
\end{algorithm}

\paragraph{Attention masks.} All layers use ordinary causal content
attention. READ rows attend to preceding and self READ rows; RECORD rows
attend to the READ rows and to preceding and self RECORD rows; WRITE rows
attend to the READ rows, all valid RECORD rows, and preceding and self WRITE
rows. $C$ RECORD slots are reserved even for a partial segment, and padded
records are masked as attention keys.

\paragraph{Boundaries and deferred writes.} A boundary occurs after $C$
admitted records, counting the initial observation and valid reset
observations. Cached execution defers the write until the next record
arrives: the full segment remains in state, then the model writes, clears
the cache, and loads the new READ tokens before admitting that record.
Ending an outer task alone triggers no write. Attempt resets preserve
memory; an outer-task reset clears the cache and restores $Z_{\rm init}$
independently for each agent. Reset-only steps keep the interface's reward
and transition semantics and are not physical actions.

\paragraph{Training and rebuild.} Dense training reconstructs complete replayed tasks under current weights
and propagates the actor--critic loss
(Appendix~\ref{app:learning-objective}) through every carried write.
Only valid recorded transitions contribute to the loss; a final write with
no following segment has no downstream decision loss. Padding preserves
WRITE positions and is masked as attention keys.

Before collection resumes, \textsc{Rebuild} runs without gradients in
evaluation mode on $U^{\rm proc}$, the records consumed since the agent's
last outer-task reset. This prefix excludes the pending observation.
Using $q=\lfloor(n-1)/C\rfloor$ preserves the deferred write when the
last processed segment is full. The resulting state matches cached
execution under current weights. Rebuilding avoids replaying stored recurrent states,
which become stale as the weights change \citep{kapturowski2019recurrent}. Frozen-policy evaluation uses only
Algorithm~\ref{alg:rsm-step}, with no replay or reconstruction.

\par\endgroup

\section{Off-Policy Learning Objective}
\label{app:learning-objective}

All conditions use AMAGO v3.4.0's scalar actor--critic learner
\citep{grigsby2024amago}. The actor and four critic heads share the timestep
and trajectory encoders. Losses average over valid replay positions and
discount heads, with padding masked. Slowly updated target networks are
maintained for the actor and critic heads.

\paragraph{Critic regression.}
For discount $\gamma$, let $Q_{\psi,i}^\gamma$ denote critic $i$ in scaled
reward units and
$q_{\psi,i}^\gamma=(Q_{\psi,i}^\gamma-\mu_\gamma)/\sigma_\gamma$ its
PopArt-normalised value \citep{vanhasselt2016learning}. With reward scale
$\alpha_r$ and a pair of target-critic indices $\mathcal I_{t,\gamma}$
sampled uniformly with replacement, the one-step target is
\begin{equation}
 y_t^\gamma=\alpha_r r_{t+1}+\gamma(1-d_{t+1})
 \min_{i\in\mathcal I_{t,\gamma}}
 \sum_{a'\in\mathcal A}\bar\pi^\gamma(a'\mid m_{t+1})
 Q_{\bar\psi,i}^\gamma(m_{t+1},a'),
 \label{eq:amago-target}
\end{equation}
where bars denote slowly updated target heads, and the critic loss is
\begin{equation}
 \mathcal L_Q=
 \E_{\tau,t,\gamma,i}
 \left[\left(q_{\psi,i}^\gamma(m_t,a_t)-
 \operatorname{sg}\!\left(
 \frac{y_t^\gamma-\mu_\gamma}{\sigma_\gamma}\right)\right)^2\right].
 \label{eq:amago-critic}
\end{equation}
The operator $\operatorname{sg}$ stops gradients. Only the outer-task
terminal $d$ stops bootstrapping; an attempt reset preserves future value.

\paragraph{Policy improvement.}
With the current policy value under critic $i$ and the ensemble-mean
advantage of a replay action,
\begin{align}
 v_{i,t}^\gamma
 &=\sum_{a\in\mathcal A}\pi_\theta^\gamma(a\mid m_t)
 q_{\psi,i}^\gamma(\operatorname{sg}(m_t),a),\nonumber\\
 A_t^\gamma
 &=\frac{1}{4}\sum_{i=1}^4
 \left[q_{\psi,i}^\gamma(m_t,a_t)-v_{i,t}^\gamma\right],
 \label{eq:amago-advantage}
\end{align}
AMAGO combines value maximisation with behaviour cloning filtered by a
positive advantage \citep{wang2020critic}:
\begin{equation}
 \mathcal L_\pi=-\E_{\tau,t,\gamma}
 \left[c_Q\min_i v_{i,t}^\gamma+
 c_{\rm BC}\operatorname{sg}\!\left(\mathbf 1\{A_t^\gamma>0\}\right)
 \log\pi_\theta^\gamma(a_t\mid m_t)\right].
 \label{eq:amago-actor}
\end{equation}
We optimise $\mathcal L=\mathcal L_\pi+\lambda_Q\mathcal L_Q$ with
$c_Q=1$, $c_{\rm BC}=0.1$, $\lambda_Q=10$, and $\alpha_r=100$.
Reward scaling applies to learning targets; policy feedback and evaluation
metrics use native, unscaled rewards.

\paragraph{Discount heads.}
The learner uses seven discount heads,
$\Gamma=\{0.1,\allowbreak 0.9,\allowbreak 0.95,\allowbreak 0.97,\allowbreak 0.99,\allowbreak 0.995,\allowbreak 0.999\}$, and the policy acts with
the $\gamma=0.999$ head. Both losses backpropagate through the shared
causal encoder over complete outer tasks, including every carried write;
READ and WRITE slots have no separate loss.

\paragraph{Gradient through carried summaries.}
Let $\mathcal L_b$ be the part of $\mathcal L$ on the decisions of segment
$b$ and $\mathcal L_{>b}=\sum_{b'>b}\mathcal L_{b'}$. Because $\mathcal L_{b+1}$
reads $Z_{b+1}$ directly and every later loss reads it only through
$Z_{b+2}$, the gradient that reaches a summary obeys
\begin{equation}
 \frac{\partial\mathcal L_{>b}}{\partial Z_{b+1}}
 =\frac{\partial\mathcal L_{b+1}}{\partial Z_{b+1}}
 +\Bigl(\frac{\partial Z_{b+2}}{\partial Z_{b+1}}\Bigr)^{\!\top}
 \frac{\partial\mathcal L_{>b+1}}{\partial Z_{b+2}},
 \label{eq:summary-gradient}
\end{equation}
and trains the writer of segment $b$ through $Z_{b+1}=P_\theta(W_b)$. Under
the overwrite,
$\partial Z_{b+2}/\partial Z_{b+1}=\partial P_\theta(W_{b+1})/\partial Z_{b+1}$
passes only through attention to the READ rows; under \method{}-R it is
$I+\partial P_\theta(W_{b+1})/\partial Z_{b+1}$.

\section{Memory Bounds}
\label{app:memory-analysis}

\subsection{Deployment-state accounting}
\label{app:cost}

\begin{proof}[Proof of Proposition~\ref{prop:bounded-state}]
At reset, each agent stores $Md$ summary values and per-layer key/value
slabs with $C+2M$ slots. A READ pass uses $M$ slots, followed by at most
$C$ RECORD slots. Before the next record, $M$ WRITE queries occupy the
reserved final slots and produce the replacement summary. Clearing the
cache and loading that summary restores the same layout. Induction over
records and boundaries gives $Md+2L(C+2M)d$ retained floating-point
values, regardless of the number of segments. FP32 storage and the
32-bit and 64-bit counters yield Equation~\ref{eq:state-bound}, within
the supported counter range. Residual and no-carry updates change values
but preserve this allocation.
\end{proof}

For $C=32$, $M=4$, $d=256$, and $L=3$, the per-agent allocation is
\begin{equation}
 4(4\cdot256)+4(2\cdot3\cdot40\cdot256)+12
 =249{,}868\ \text{bytes}.
\end{equation}
One detached FP32 copy of $Z_{\rm init}$, $4Md=4{,}096$ bytes, is shared
across the agent batch, so $N$ agents retain $N B_{\rm deploy}+4Md$ tensor
bytes, excluding parameters, temporary computation, replay, and
training-time prefixes and gradients.

\subsection{State of the compared memories}
\label{app:state-others}

Table~\ref{tab:state} gives the state that each memory keeps per agent
between decisions, counted as for Proposition~\ref{prop:bounded-state}: the
history-dependent tensors at FP32 with their counters, excluding parameters
and workspace, at their peak within a task of $T$ steps. Table~\ref{tab:results}
and Figure~\ref{fig:state} report these peaks, measured at $T=4{,}000$. The
fixed-size memories keep the same state at every step; the full history and
Memo grow with every record or segment.

\begin{table}[H]
\centering\small
\caption{Per-agent state of each memory ($L=3$, $d=256$, FP32).}
\label{tab:state}
\begin{tabular}{lp{5.4cm}r}
\toprule
Memory & State kept between decisions & Peak, $T=4{,}000$ \\
\midrule
\method{}, \method{}-R, \nocarry{} & Summary, buffer caches and counters (Equation~\ref{eq:state-bound}) & 249.9\,KB \\
Sliding window & Keys and values of $W$ records per layer, their positions and a length: $8LWd+8W+4$ bytes & 246.1\,KB \\
GRU & Hidden state of three layers: $12d$ bytes & 3.1\,KB \\
Memo, both variants & Keys and values of four tokens per completed segment and of the open segment & 3.24\,MB \\
Full-history Transformer & Keys and values of every record and a length: $8L(T+1)d+4$ bytes & 24.58\,MB \\
\bottomrule
\end{tabular}
\end{table}

\subsection{Bounded summary values}
\label{app:summary-range}

\begin{proposition}[Bounded summary values]
\label{prop:summary-range}
Fix $\theta$, and let
$\nu=\sup_w\|P_\theta(w)\|_2$ over final-normalised rows $w$, which is
finite because the final normalisation is a LayerNorm
\citep{ba2016layer}. For every task, every $b\geq1$ and every row
$j\in\{1,\ldots,M\}$,
\begin{equation}
 \|Z_b[j]\|_2\leq\nu\quad\text{(\method{})},\qquad
 \|Z_b[j]\|_2\leq\|Z_{\rm init}[j]\|_2+b\,\nu\quad\text{(\method{}-R)}.
 \label{eq:summary-range}
\end{equation}
\end{proposition}
Under the overwrite, the READ input thus stays in one bounded set however
many boundaries a task crosses. The residual rule has only a
boundary-dependent bound, which grows with every write. These norm bounds
concern the summary's values, not what they encode or how the policy uses
them.

\begin{proof}[Proof of Proposition~\ref{prop:summary-range}]
The final LayerNorm maps a row $x\in\R^d$ to $w=\eta\odot\hat x+\zeta$,
with gain $\eta$, bias $\zeta$ and
$\hat x_k=(x_k-\bar x)/\sqrt{s^2+\epsilon}$, where $\bar x$ and $s^2$ are
the row's mean and (biased) variance and $\epsilon>0$. Since
$\sum_k(x_k-\bar x)^2=d s^2$, we have
$\|\hat x\|_2^2=d s^2/(s^2+\epsilon)<d$, so
$\|w\|_2\leq\sqrt d\,\|\eta\|_\infty+\|\zeta\|_2$ for every input. Writing
$P_\theta(w)=A_{\rm P}w+c_{\rm P}$ on each row,
\begin{equation}
 \nu\leq\|A_{\rm P}\|_2\bigl(\sqrt d\,\|\eta\|_\infty+\|\zeta\|_2\bigr)
 +\|c_{\rm P}\|_2<\infty.
\end{equation}
Under the overwrite, $Z_b[j]=P_\theta(W_{b-1}[j])$ for $b\geq1$, so
$\|Z_b[j]\|_2\leq\nu$. Under the residual rule,
$Z_b=Z_{\rm init}+\sum_{b'=0}^{b-1}P_\theta(W_{b'})$, and the triangle
inequality gives $\|Z_b[j]\|_2\leq\|Z_{\rm init}[j]\|_2+b\,\nu$. Both
bounds hold for every input sequence, so they apply at every horizon.
\end{proof}

\section{Evaluation Protocol}
\label{app:benchmarks}

Every policy observes the fields listed below together with its previous
action and reward; task identities, hidden layouts and evaluator labels
never reach it. Training tasks are drawn from identities $0$--$999{,}999$,
development tasks from $1{,}000{,}000$--$1{,}000{,}063$ and the 256
confirmation tasks from $5{,}000{,}000$--$5{,}000{,}255$, disjoint from
each other. Every reported result is one greedy rollout per confirmation task
at the final checkpoint of each run, and nothing is retrained or tuned for the
long-running evaluations.

\subsection{Benchmarks}
\label{app:tasks}

\paragraph{Dark Key-to-Door.} The start, key and door cells of an
$8\times8$ room are hidden and fixed for the outer task; the agent observes
its position, whether it holds the key, and the attempt clock. A task lasts
500 environment steps and holds attempts of at most 50 actions separated by
reset-only steps, with memory kept across attempts. The metric is the
number of doors completed in 500 steps ($\dlt=1$ door).

\paragraph{CountRecall.} A stream of 104 cards from four categories is
dealt in a hidden order; the agent observes the dealt and the queried
category (one-hot) and a timer, and answers 103 scored queries for the
running count of the queried category. The metric is exact accuracy
($\dlt=0.05$). Retention is read by the number of summary rewrites since
the queried category was first dealt.

\paragraph{MazeRunner.} The walls of a $15\times15$ maze and the
permutation of its five actions are hidden; the agent observes its
position, the distances to the four surrounding walls, a timer and the
coordinates of three ordered goals, all scaled by the maze size. A task is
one episode of at most 500 steps that ends at the third goal or at the
timer. The metric is the fraction of goals reached ($\dlt=0.1$; one goal
per 500 steps on the long-running tasks).

\paragraph{Passive T-Maze.} A cue at the first step fixes the rewarded turn
at the junction of a corridor of length $L$; the agent observes the
cue only then; afterwards it observes a flag that marks the junction. Training draws $L$ from the multiples
of 32 between 32 and 256 and allows $L+1$ decisions, with a $-1/128$
penalty for a non-forward action before the junction. The metric is
success at $L=128$ ($\dlt=0.25$).

\subsection{Long-running evaluations}
\label{app:long-protocol}

Each frozen policy keeps acting in one task for 4,000 steps with its memory
carried throughout, and we count its successes in each 500-step window, so a
longer task confers no advantage from extra opportunity alone. In the
\emph{same task}, the Key-to-Door layout stays fixed while attempts continue,
and one MazeRunner map is replayed over repeated attempts: after the third
goal or the 500-step timer, the agent restarts on the same maze. A
MazeRunner training task holds a single attempt, so on MazeRunner every
attempt after the first, in both long-running protocols, is new to every
memory, whereas Key-to-Door trains with repeated attempts. In a
\emph{new task inside the run}, a new Key-to-Door layout is drawn at the first
attempt boundary after every 500 steps, and a new MazeRunner maze and goal
sequence, with the same action permutation, at the first attempt boundary
after every 1,000 steps; the new maps are drawn from identities
$6{,}000{,}000$--$6{,}999{,}999$. Wiping a memory at every attempt boundary
isolates what it carries from one attempt to the next.

\subsection{Statistics and thresholds}
\label{app:statistics}

We report means over tasks and seeds. A paired difference between two
memories is formed on each task and seed; its 95\% interval comes from a joint
bootstrap that resamples the three seeds and the tasks with replacement, 2,000
times \citep{agarwal2021precipice}, and we also report the difference within
each seed. CountRecall resamples whole streams, never the queries within one.
A difference is \emph{consistent} when its mean exceeds the benchmark's
margin $\dlt$ in magnitude, its joint interval excludes zero and all three
seeds agree in sign,
\emph{practically equivalent under this rule} when the whole interval lies
within $\pm\dlt$, and otherwise no consistent difference. These are
descriptive reporting labels, not additional hypothesis tests.
Tables~\ref{tab:contrasts} and~\ref{tab:contrasts-long} list the
paired differences between \method{} and the other memories behind the
comparisons of Section~\ref{sec:experiments}, at the training horizon, over
long tasks, after task changes and across summary lengths, each with its
interval, its value on each seed and its verdict. The margins are practical
reporting thresholds rather than significance levels: one event per 500-step
window on Key-to-Door and long-running MazeRunner, five accuracy points on
CountRecall, ten percentage points of the three-goal fraction on MazeRunner,
and 25 percentage points of success on T-Maze. The development thresholds of
Figure~\ref{fig:learning} are 20
doors per task on Key-to-Door, 0.9 exact accuracy on CountRecall, a goal
fraction of 0.5 on MazeRunner (descriptive only) and T-Maze success of 1.0
held to the end of training, each on the 64 development tasks. They summarise
learning speed and are not used for checkpoint selection.

\begin{table}[!tp]
\centering
\caption{\textbf{Paired differences at the training horizon}, \method{} minus each memory on the same tasks
and seeds: mean, 95\% joint seed/task interval, the value on each seed
(42 / 100 / 2026) and the verdict of Appendix~\ref{app:statistics}. T-Maze: every
memory reaches 1.000 on every seed (practically equivalent under our rule), except \nocarry{}
($+0.486$ $[+0.410, +0.536]$; consistent). At the Key-to-Door and CountRecall training
horizon, \method{}-R differs from \method{} by less than $\dlt$ in mean (Table~\ref{tab:results}), so
neither difference is consistent.}
\label{tab:contrasts}
\footnotesize
\setlength{\tabcolsep}{2.5pt}
\begin{tabular}{lcccl}
\toprule
\method{} minus & Mean & 95\% interval & Seeds & Verdict \\
\midrule
\multicolumn{5}{l}{\emph{Key-to-Door, training horizon} ($\dlt=1$)} \\
Full-history Transformer & $+1.6$ & $[-1.3, +4.7]$ & $-0.5$ / $+3.8$ / $+1.6$ & no consistent difference \\
GRU & $-0.5$ & $[-5.0, +3.3]$ & $-4.9$ / $+0.4$ / $+2.8$ & no consistent difference \\
Memo & $-0.5$ & $[-2.8, +1.7]$ & $+0.5$ / $-0.3$ / $-1.8$ & no consistent difference \\
Memo, fixed segments & $+2.3$ & $[-1.4, +5.5]$ & $-1.1$ / $+3.9$ / $+4.1$ & no consistent difference \\
Sliding window & $+4.9$ & $[+1.4, +9.0]$ & $+2.6$ / $+3.1$ / $+9.1$ & consistent \\
\nocarry{} & $+27.0$ & $[+24.4, +29.8]$ & $+26.4$ / $+29.2$ / $+25.5$ & consistent \\
\midrule
\multicolumn{5}{l}{\emph{CountRecall, training horizon} ($\dlt=0.05$)} \\
Full-history Transformer & $-0.009$ & $[-0.012, -0.007]$ & $-0.010$ / $-0.010$ / $-0.008$ & equivalent \\
GRU & $+0.059$ & $[+0.017, +0.095]$ & $+0.016$ / $+0.065$ / $+0.096$ & consistent \\
Memo & $+0.033$ & $[+0.025, +0.042]$ & $+0.043$ / $+0.032$ / $+0.024$ & equivalent \\
Memo, fixed segments & $+0.045$ & $[+0.037, +0.057]$ & $+0.058$ / $+0.038$ / $+0.039$ & no consistent difference \\
\nocarry{} & $+0.309$ & $[+0.292, +0.330]$ & $+0.295$ / $+0.303$ / $+0.331$ & consistent \\
\method{} detach & $+0.075$ & $[+0.062, +0.088]$ & $+0.078$ / $+0.088$ / $+0.060$ & consistent \\
\midrule
\multicolumn{5}{l}{\emph{MazeRunner, training horizon} ($\dlt=0.1$)} \\
Full-history Transformer & $+0.067$ & $[+0.038, +0.098]$ & $+0.060$ / $+0.065$ / $+0.077$ & equivalent \\
GRU & $+0.161$ & $[+0.123, +0.199]$ & $+0.172$ / $+0.169$ / $+0.142$ & consistent \\
Memo & $-0.059$ & $[-0.103, -0.017]$ & $-0.066$ / $-0.025$ / $-0.087$ & no consistent difference \\
Memo, fixed segments & $-0.033$ & $[-0.073, +0.005]$ & $-0.010$ / $-0.061$ / $-0.026$ & equivalent \\
Sliding window & $+0.053$ & $[+0.002, +0.107]$ & $+0.010$ / $+0.102$ / $+0.048$ & no consistent difference \\
\nocarry{} & $+0.066$ & $[+0.033, +0.105]$ & $+0.049$ / $+0.060$ / $+0.087$ & no consistent difference \\
\method{}-R & $+0.048$ & $[+0.016, +0.084]$ & $+0.061$ / $+0.052$ / $+0.031$ & equivalent \\
\bottomrule
\end{tabular}
\end{table}

\begin{table}[!tp]
\centering
\caption{\textbf{Paired differences on long-running tasks and summary-length
ablations}, reported as in Table~\ref{tab:contrasts}.}
\label{tab:contrasts-long}
\footnotesize
\setlength{\tabcolsep}{2.5pt}
\begin{tabular}{lcccl}
\toprule
\method{} minus & Mean & 95\% interval & Seeds & Verdict \\
\midrule
\multicolumn{5}{l}{\emph{Key-to-Door, last 500 of 4,000 steps} ($\dlt=1$)} \\
Full-history Transformer & $+17.8$ & $[+13.1, +21.5]$ & $+20.8$ / $+19.8$ / $+12.8$ & consistent \\
GRU & $+5.4$ & $[+1.9, +8.5]$ & $+1.6$ / $+6.8$ / $+7.7$ & consistent \\
Memo & $+35.4$ & $[+27.0, +42.3]$ & $+37.9$ / $+26.0$ / $+42.2$ & consistent \\
Memo, fixed segments & $+14.6$ & $[+10.0, +18.8]$ & $+17.4$ / $+10.0$ / $+16.4$ & consistent \\
Sliding window & $+41.1$ & $[+38.6, +43.8]$ & $+42.4$ / $+41.6$ / $+39.3$ & consistent \\
\nocarry{} & $+43.1$ & $[+40.6, +45.6]$ & $+42.8$ / $+43.8$ / $+42.7$ & consistent \\
\method{}-R & $+0.7$ & $[-0.7, +2.1]$ & $-0.3$ / $+0.7$ / $+1.8$ & no consistent difference \\
\midrule
\multicolumn{5}{l}{\emph{MazeRunner, last 500 of 4,000 steps} ($\dlt=1$)} \\
Full-history Transformer & $+25.8$ & $[+22.5, +29.3]$ & $+25.5$ / $+24.8$ / $+27.2$ & consistent \\
GRU & $+7.6$ & $[+5.7, +9.5]$ & $+7.6$ / $+6.8$ / $+8.2$ & consistent \\
Memo & $+24.8$ & $[+21.0, +28.8]$ & $+24.8$ / $+22.0$ / $+27.6$ & consistent \\
Memo, fixed segments & $+23.8$ & $[+20.1, +27.5]$ & $+24.0$ / $+21.6$ / $+25.9$ & consistent \\
Sliding window & $+24.1$ & $[+21.2, +27.2]$ & $+23.2$ / $+23.9$ / $+25.1$ & consistent \\
\nocarry{} & $+7.4$ & $[+4.6, +10.3]$ & $+4.9$ / $+7.7$ / $+9.7$ & consistent \\
\method{}-R & $+21.6$ & $[+18.4, +25.1]$ & $+19.1$ / $+23.7$ / $+22.1$ & consistent \\
\midrule
\multicolumn{5}{l}{\emph{Key-to-Door, doors per later layout} ($\dlt=1$)} \\
Full-history Transformer & $+8.6$ & $[+6.9, +10.3]$ & $+6.9$ / $+10.4$ / $+8.5$ & consistent \\
GRU & $+7.8$ & $[+5.6, +9.9]$ & $+5.6$ / $+10.0$ / $+7.9$ & consistent \\
Memo & $+8.9$ & $[+7.3, +10.5]$ & $+7.2$ / $+10.5$ / $+9.2$ & consistent \\
Memo, fixed segments & $+8.8$ & $[+7.0, +10.6]$ & $+6.9$ / $+10.6$ / $+8.8$ & consistent \\
\nocarry{} & $+7.5$ & $[+5.8, +9.1]$ & $+5.8$ / $+9.2$ / $+7.5$ & consistent \\
\method{}-R & $+7.6$ & $[+5.4, +10.2]$ & $+5.4$ / $+10.4$ / $+7.0$ & consistent \\
\midrule
\multicolumn{5}{l}{\emph{MazeRunner, goals per 500 steps on the fourth map} ($\dlt=1$)} \\
Full-history Transformer & $+22.6$ & $[+19.9, +25.3]$ & $+23.2$ / $+21.1$ / $+23.3$ & consistent \\
GRU & $+6.9$ & $[+4.9, +9.0]$ & $+7.9$ / $+5.7$ / $+7.1$ & consistent \\
Memo & $+21.0$ & $[+17.5, +24.4]$ & $+21.9$ / $+17.9$ / $+23.3$ & consistent \\
Memo, fixed segments & $+20.0$ & $[+17.1, +23.1]$ & $+20.9$ / $+17.7$ / $+21.6$ & consistent \\
\nocarry{} & $+5.1$ & $[+2.9, +7.7]$ & $+3.6$ / $+4.4$ / $+7.3$ & consistent \\
\method{}-R & $+22.6$ & $[+19.8, +25.6]$ & $+23.6$ / $+20.9$ / $+23.4$ & consistent \\
\midrule
\multicolumn{5}{l}{\emph{Key-to-Door, summary length: \method{} ($M=4$) minus $M$ tokens} ($\dlt=1$)} \\
$M=1$, steps 1--500 & $-0.5$ & $[-4.0, +2.7]$ & $-4.0$ / $+2.0$ / $+0.3$ & no consistent difference \\
$M=1$, steps 3,501--4,000 & $+0.7$ & $[-1.9, +2.8]$ & $-1.9$ / $+1.5$ / $+2.4$ & no consistent difference \\
$M=8$, steps 1--500 & $-1.8$ & $[-4.7, +1.4]$ & $-2.5$ / $+1.2$ / $-4.0$ & no consistent difference \\
$M=8$, steps 3,501--4,000 & $+0.7$ & $[-1.2, +2.9]$ & $-0.3$ / $+3.1$ / $-0.7$ & no consistent difference \\
$M=16$, steps 1--500 & $-1.0$ & $[-3.7, +1.7]$ & $-1.2$ / $+1.0$ / $-2.9$ & no consistent difference \\
$M=16$, steps 3,501--4,000 & $+0.1$ & $[-0.8, +1.3]$ & $+0.6$ / $-0.1$ / $-0.3$ & no consistent difference \\
\bottomrule
\end{tabular}
\end{table}

\section{Additional Results}
\label{app:experiments}

Appendices~\ref{app:in-context} and~\ref{app:extended} extend
Section~\ref{sec:results-endpoint};
Appendices~\ref{app:interventions}--\ref{app:attention} extend
Section~\ref{sec:results-mechanism}; Appendix~\ref{app:ablations} extends
Section~\ref{sec:ablations}; and Appendix~\ref{app:learning} shows the
training runs of Section~\ref{sec:setup}.

\subsection{In-context adaptation at the training horizon}
\label{app:in-context}

Figure~\ref{fig:in-context} shows adaptation within one task at the training
horizon. On Key-to-Door every memory except \nocarry{} raises its door success
from about 0.15 on the first attempt to above 0.8 on the ninth by reusing the
layout it found, while \nocarry{} stays near its first-attempt rate. On
CountRecall \method{}'s cumulative accuracy stays close to the full history's
through all three rewrites, whereas \nocarry{} falls once queries reach beyond
its buffer and both Memo variants drift down late in the stream. On MazeRunner
the fraction of maps reaching each ordered goal is similar across memories,
with the GRU lowest. On the T-Maze every memory reaches the goal the cue
indicates on every task and on both cue sides, whereas \nocarry{} turns the
same way on every task and succeeds only when the cue points that way.

\begin{figure}[H]
 \centering
 \includegraphics[width=\linewidth,trim=0bp 0bp 0bp 2bp,clip]{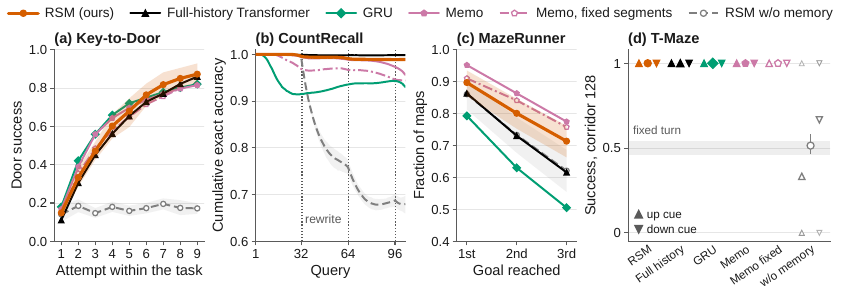}
 \caption{\textbf{Within-task adaptation.}
(a) Key-to-Door door success by attempt. (b) CountRecall cumulative accuracy;
dotted: rewrites after queries 32, 64, 96. (c) Fraction of MazeRunner maps
reaching each ordered goal. (d) T-Maze success at corridor 128 over all tasks
(95\% interval) and by cue side (triangles; faint: seeds); shaded: the success
of a fixed turn. Bands in (a--c): \method{} and \nocarry{}.}
 \label{fig:in-context}
\end{figure}

\subsection{Long-running evaluations}
\label{app:extended}

Table~\ref{tab:wipe} separates what each memory carries from one attempt to
the next on one MazeRunner map (Appendix~\ref{app:long-protocol}).
The retained-memory mean for \method{} exceeds its wiped-memory mean in both
the first and last windows, by similar amounts. For both Memo variants, the
retained-memory mean falls to a few goals per window, whereas the wiped-memory
mean stays near its early value. The full-history Transformer declines even
when wiped, so its retained history alone cannot explain the decline, and its
retained-versus-wiped difference is negligible in the last window.
\mbox{\method{}-R} shows the same retained-memory decline
(Appendix~\ref{app:ablations}), while the GRU and \nocarry{} means change
little. Table~\ref{tab:wipe} reports point estimates; we therefore treat these
kept-versus-wiped comparisons as descriptive.

\begin{table}[t]
\centering\small
\caption{\textbf{Memory kept or wiped at every attempt} on one MazeRunner map:
goals in the first and last 500 of 4,000 steps (confirmation maps, three
seeds).}
\label{tab:wipe}
\begin{tabular}{lcccc}
\toprule
 & \multicolumn{2}{c}{Memory kept} & \multicolumn{2}{c}{Memory wiped} \\
\cmidrule(lr){2-3}\cmidrule(lr){4-5}
 & 1--500 & 3,501--4,000 & 1--500 & 3,501--4,000 \\
\midrule
\method{} (ours) & 25.0 & 26.6 & 21.4 & 22.2 \\
Full-history Transformer & 8.4 & 0.8 & 3.6 & 0.7 \\
GRU & 17.6 & 19.1 & 14.9 & 15.7 \\
Memo & 19.5 & 1.8 & 21.1 & 22.2 \\
Memo, fixed segments & 18.5 & 2.8 & 18.2 & 18.7 \\
\midrule
\nocarry{} & 18.4 & 19.2 & 17.7 & 18.2 \\
\method{}-R (residual) & 20.8 & 5.0 & 20.4 & 21.2 \\
\bottomrule
\end{tabular}
\end{table}

Figure~\ref{fig:new-tasks} resolves the new-task runs of
Figure~\ref{fig:at-a-glance}g,h in time. With memory never reset,
\method{} is the only memory whose mean success rate rises within later
Key-to-Door layouts: averaged over tasks, its rate within each later layout climbs from
under one door per 100 steps to about 3.5, for 9.4 doors per layout against
at most 1.9 for every other memory. This average comes from a minority of
layouts. \method{} completes a door in 34\% of later layouts, ranging from
22\% to 48\% across seeds, and then repeats it at about its adapted rate, but
completes none in the rest. This asymmetry suggests that information retained
from the preceding layout can bias subsequent search, but the experiment does
not identify the content responsible. The other memories that carry the old layout complete a
door in at most 15\% of later layouts, so in panel (c), where a layout
without a door counts as 500 steps, all of them sit near that limit (about
410 steps for \method{}, at least 450 for the others). Only \nocarry{}, which
carries no stale layout, finds a door in most later layouts (72\%, after
about 270 steps), but it cannot repeat what it finds. On MazeRunner, \method{} finishes its first attempt
on each new map in about 200 steps, as on the first map, against about 300
for the GRU and at least 370 for the growing memories.

\begin{figure}[H]
 \centering
 \includegraphics[width=\linewidth,trim=0bp 0bp 0bp 2bp,clip]{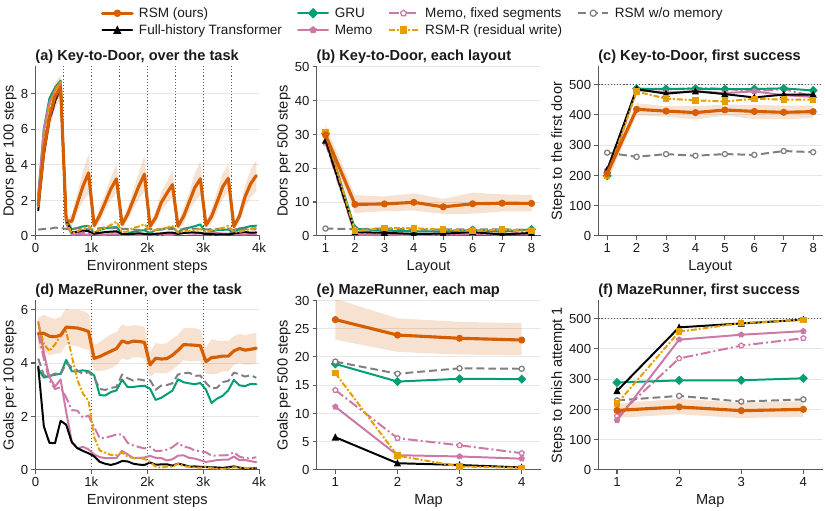}
 \caption{\textbf{Re-adaptation without memory resets} (4,000 steps).
New Key-to-Door layouts every 500 steps (top); new MazeRunner mazes/goals
every 1,000 (bottom). (a,d) Successes per 100 steps; changes dotted.
(b,e) Doors per layout/goals per 500 steps. (c,f) Steps to first door
(censored at 500)/completed attempt. Band: \method{}.}
 \label{fig:new-tasks}
\end{figure}

\subsection{Summary transplants}
\label{app:interventions}

At one boundary, we replace every task's written summary with the summary
that another task of the same seed read at that boundary, and leave the
buffer, the weights and everything else unchanged; the control writes the
initial summary at the same boundary (\emph{cleared once}). On Key-to-Door,
the first attempt after the boundary visits the donor's key cell far more
often than with the task's own summary or with the summary cleared once
(Figure~\ref{fig:transplant}b), and the misdirected policy completes fewer
doors than the cleared one for the rest of the task
(Figure~\ref{fig:transplant}a): the transplanted summary redirects the policy
towards donor-associated cells, and the effect persists across later rewrites.
On CountRecall the
transplanted policy answers with the counts that the donor's summary implies
rather than the true counts until the next rewrite, which moves its answers
away from the donor's counts and partly towards the task's own
(Figure~\ref{fig:transplant}c). Because the early cards are not presented
again, subsequent rewards are the only new feedback about whether answers
that depend on those early counts are correct.

\begin{figure}[H]
 \centering
 \includegraphics[width=\linewidth,trim=0bp 0bp 0bp 2bp,clip]{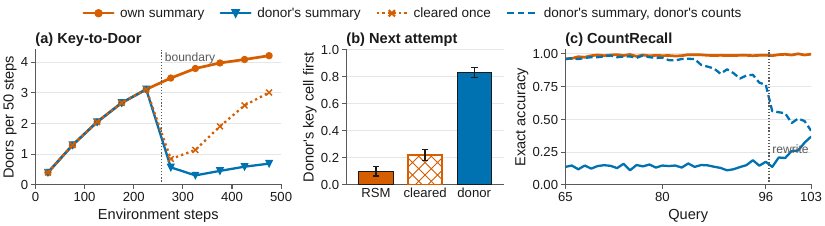}
 \caption{\textbf{Summary transplants.}
At one boundary, the summary is replaced by a donor task's (blue) or the
initial summary (dotted). (a) Key-to-Door doors per 50 steps; transplant before
step 257. (b) Fraction of tasks whose next attempt visits the donor's key
first. (c) CountRecall accuracy after query 64, scored against true (solid)
or donor-consistent (dashed) counts.}
 \label{fig:transplant}
\end{figure}

\subsection{What the summary holds}
\label{app:representation}

Figure~\ref{fig:representation-detail} looks inside the frozen summary
(Section~\ref{sec:results-mechanism} and Section~\ref{sec:ablations}). In
panels (a--c) and (e--g), one per seed, each dot is one confirmation task,
placed by the summary its policy reads at one boundary and coloured by a hidden
variable of that task; panels (d,h) report how well a linear read-out recovers
each hidden variable from the summary. Ridge
regressions fitted on the 64 development tasks and scored on the 256
confirmation tasks recover every count on CountRecall almost perfectly and the
row and column of the hidden key on Key-to-Door well, on every seed; the
door's position is recovered less well. Projected on its first two principal
components, the Key-to-Door summary orders the tasks by where the key lies and
the CountRecall summary by the counts, most clearly on seed~42 (the read-outs
of the other seeds are as high). Panels (i--k) follow one MazeRunner map
over 4,000 steps of repeated attempts: (i,j) compare the summary written at
each point of the run with the summary at every other point, and (k) tracks
its size. \method{}'s summaries retain high within-map similarity at a steady
norm across rewrites, consistent with a stable map-dependent representation.
\mbox{\method{}-R} summaries instead drift in one direction until the first
and last are anti-correlated, at a norm tens of times larger.
Proposition~\ref{prop:summary-range} provides a horizon-independent norm bound
for the overwrite, but only a boundary-dependent bound for the additive
write. These diagnostics are consistent with, but do not establish the cause
of, \mbox{\method{}-R}'s decline on long MazeRunner runs
(Section~\ref{sec:ablations}). A good read-out shows that the information is
present in the summary, not that the policy uses it; the interventions of
Section~\ref{sec:results-mechanism} test its use.

\begin{figure}[H]
 \centering
 \includegraphics[width=\linewidth,trim=0bp 1bp 0bp 0bp,clip]{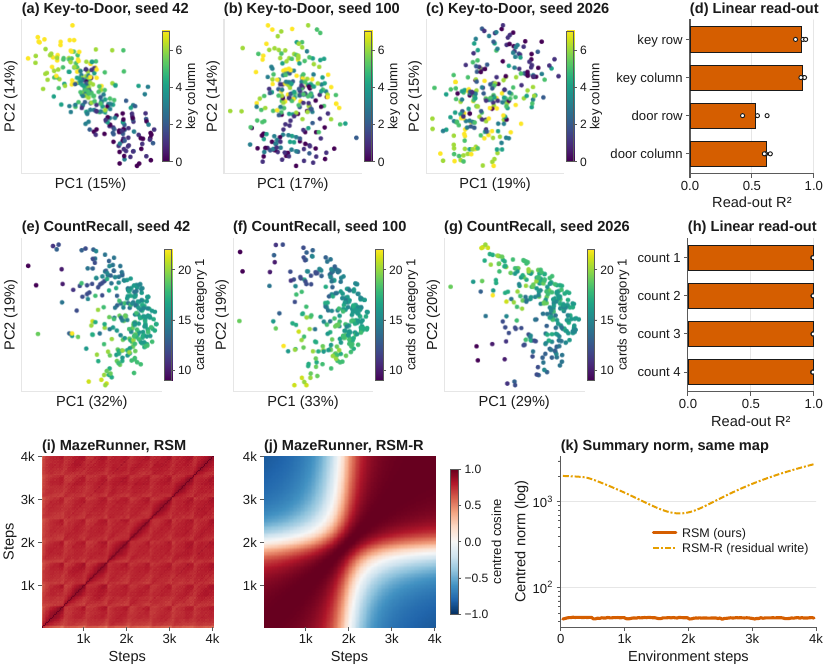}
 \caption{\textbf{Summary representations.}
(a--c,e--g) First two principal components, fitted per seed on development
tasks; colour: key column (Key-to-Door, step 257)/first-category count
(CountRecall, query 65). (d,h) Ridge $R^2$; bars: seed means; dots: seeds.
(i,j) Centred cosine similarity of the summaries over repeated MazeRunner
attempts on one map; (k) centred
norms. Panels (i--k): first 64 confirmation maps.}
 \label{fig:representation-detail}
\end{figure}

\paragraph{Probe details.} Each read-out is a ridge regression from the
flattened summary ($Md$ values) that one segment reads, at step 257 on
Key-to-Door (segment 8) and at query 65 on CountRecall (segment 2), to the
task's hidden variables: the key's and the door's row and column, or the four
running counts. It is fitted per seed on the 64 development tasks and scored
as $R^2$ on the 256 confirmation tasks. Features and targets are centred by
the development means, and the penalty is chosen by five-fold cross-validation
on the development tasks from 17 values spaced logarithmically between
$10^{-3}$ and $10^{5}$. The principal components are fitted per seed on the
same development summaries. For panels (i--k), each seed's flattened summaries
are centred by their mean over all maps and segments; panels (i,j) show the
cosine similarity between the centred summaries read by two segments of the
same map, averaged over maps and then seeds, and panel (k) their mean norm.

\subsection{Summary attention}
\label{app:attention}

The READ-slot attention share is the attention mass that a decision's heads
place on the READ slots, averaged over heads. We compare it with the share
$M/(M+i)$ under uniform attention at segment position $i$ and with
\nocarry{}, whose READ slots always hold $Z_{\rm init}$. Because \nocarry{} is
separately trained, it controls for fixed READ-slot content and position, not
for position alone. Figure~\ref{fig:attention} shows the READ-slot attention
share by position in the segment and the effect of biasing decisions'
attention away from the summary or the buffer. The decisions assign more
attention mass to these slots than \nocarry{} does, in every layer on
Key-to-Door and in the second and third layers on CountRecall, on every seed.
Against uniform attention the excess is clear only in the first layer on
Key-to-Door and the second on CountRecall; the other layers stay near or below
$M/(M+i)$ early in the segment and only slightly above it later.
The share hardly depends on where the needed evidence lives: queries whose
evidence survives only in the summary, and the first steps of later
Key-to-Door attempts, draw only marginally more summary attention than their
position predicts. For Figure~\ref{fig:room-attention}
(Section~\ref{sec:results-mechanism}), each decision's summary share minus
the seed's mean at the same segment position is averaged over the decisions
taken at each cell of the Key-to-Door room, on all 256 confirmation tasks at
the training horizon, pooling seeds; every cell holds at least 50 decisions.

\begin{figure}[H]
 \centering
 \includegraphics[width=\linewidth,trim=0bp 0bp 0bp 2bp,clip]{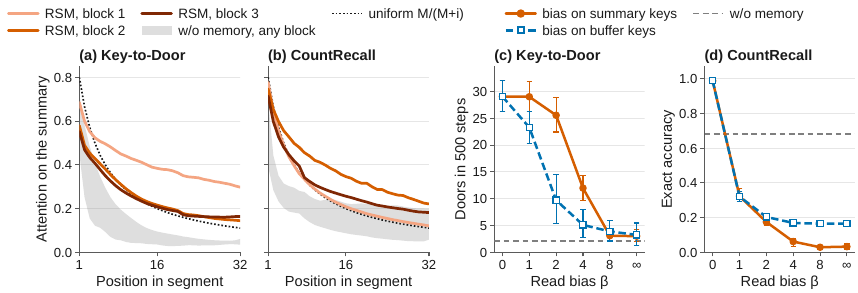}
 \caption{\textbf{Summary attention and read interventions.}
(a,b) Attention to four summary tokens by segment position/layer;
grey band: \nocarry{} (separately trained fixed-READ baseline); dotted:
uniform $M/(M+i)$.
(c,d) Performance under attention-score bias $-\beta$ on summary (solid)/earlier buffer keys
(dashed); $\beta=\infty$ blocks access; grey dashed: \nocarry{}.}
 \label{fig:attention}
\end{figure}

\subsection{Ablations}
\label{app:ablations}

Figure~\ref{fig:ablations} follows the carry and rewrite ablations over the
long-running tasks. On Key-to-Door \mbox{\method{}-R} shows no consistent
difference from \method{} over all 4,000 steps, while \nocarry{} stays near
two doors per window. On one MazeRunner map \mbox{\method{}-R} declines from
window to window until it falls below \nocarry{}; the same pattern is absent
when its memory is wiped at every attempt (Table~\ref{tab:wipe}), implicating
the carried state without identifying which property of that state causes
the decline. When the maze changes every 1,000 steps, \mbox{\method{}-R} falls
to a few goals on the second map and almost none after, whereas \method{} and
\nocarry{} keep a steady rate.

\begin{figure}[H]
 \centering
 \includegraphics[width=\linewidth,trim=0bp 0bp 0bp 2bp,clip]{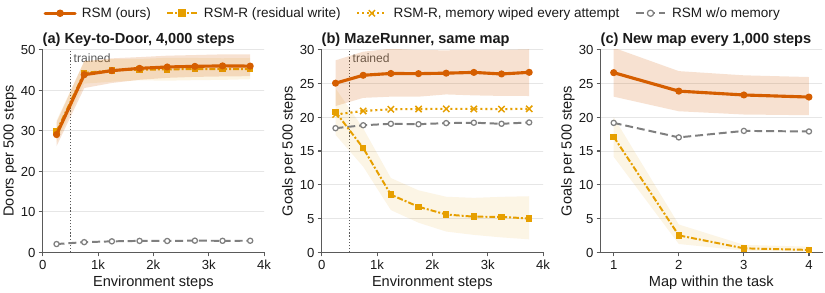}
 \caption{\textbf{Carry and rewrite ablations} (Section~\ref{sec:ablations}; 500-step windows).
(a) Key-to-Door doors. (b) MazeRunner goals on one repeated map;
crosses: RSM-R with memory wiped each attempt. (c) Goals on four maps,
changed every 1,000 steps. Dotted: training horizon.}
 \label{fig:ablations}
\end{figure}

\method{} detach is trained on CountRecall with the same recipe and seeds, its
summary detached at every boundary so that no later loss reaches the write.
\method{} exceeds \method{} detach, and \method{} detach exceeds \nocarry{};
both gaps exceed $\dlt$ on every seed and have joint intervals that exclude
zero. The summary-length arms train \method{} with
$M\in\{1,8,16\}$ on Key-to-Door under the same recipe, buffer and seeds, and
are compared with $M=4$ by paired differences on the confirmation tasks at 500
steps and in the last 500 of 4,000 steps with the layout kept
(Table~\ref{tab:summary-length}).

\subsection{Learning curves}
\label{app:learning}

Figure~\ref{fig:learning} shows collection return and greedy development
performance over the 8M training steps, with every memory trained by the same
learner and interaction budget. Every memory reaches the development threshold
(Appendix~\ref{app:statistics}) on every benchmark, except \nocarry{} on
Key-to-Door and CountRecall, and all but \nocarry{} solve the T-Maze early in
training. On Key-to-Door \method{} takes longer than every other memory that
reaches the threshold, crossing it at 4.4M to 5.2M steps across seeds against
at most 3.6M for the others, and then reaches their level by the end of
training; on MazeRunner the GRU learns slowest.

\begin{figure}[H]
 \centering
 \includegraphics[width=\linewidth,trim=0bp 0bp 0bp 2bp,clip]{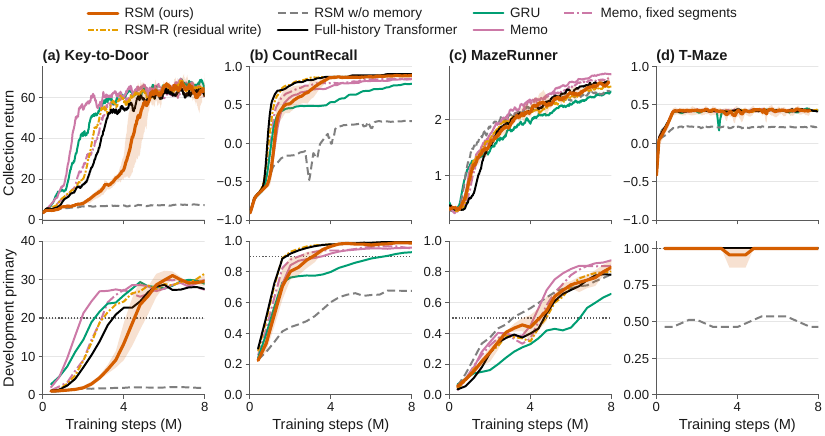}
 \caption{\textbf{Learning at matched experience.}
Collection return (top)/greedy development performance (bottom) versus
training steps for every memory on all four benchmarks. Lines: seed means;
band: \method{}'s seed range;
dotted: development thresholds. For display only, collection curves use a
centred 25-epoch window and development curves a centred three-evaluation
window, both with edge padding; threshold crossings use the raw series.}
 \label{fig:learning}
\end{figure}

\end{document}